\documentclass{article} 
\usepackage{iclr2022_conference,times}

\usepackage{amsmath,amsfonts,bm}

\def\eqref#1{equation~\ref{#1}}

\def\1{\bm{1}}

\DeclareMathAlphabet{\mathsfit}{\encodingdefault}{\sfdefault}{m}{sl}
\SetMathAlphabet{\mathsfit}{bold}{\encodingdefault}{\sfdefault}{bx}{n}

\DeclareMathOperator*{\argmax}{arg\,max}

\usepackage{url}
\usepackage{soul}

\usepackage{amsfonts,amsmath,amssymb,amsthm,mathtools}
\usepackage[colorlinks=true,allcolors=blue]{hyperref}%

\usepackage{xcolor,colortbl}
\usepackage{color}
\usepackage{tcolorbox}
\usepackage{adjustbox}
\usepackage{tabularx}

\DeclareMathOperator*{\vect}{vec}
\DeclareMathOperator*{\subopt}{SubOpt}

\usepackage{accents}

\usepackage{algorithm}
\usepackage{algpseudocode}

\theoremstyle{plain}
\newtheorem{thm}{Theorem}[section]
\newtheorem{lem}{Lemma}[section]

\theoremstyle{definition}
\newtheorem{defn}{Definition}[section]

\newtheorem{assumption}{Assumption}[section]

\theoremstyle{remark}
\newtheorem{rem}{Remark}[section]

\usepackage{xcolor}
\usepackage{enumitem}   
\usepackage{bm}
\usepackage{subcaption}
\allowdisplaybreaks
\usepackage{tablefootnote}
\usepackage{longtable}

\title{Offline Neural Contextual Bandits: \\Pessimism, Optimization and Generalization}

\author{Thanh Nguyen-Tang \thanks{Email: \url{nguyent2792@gmail.com}.} \\
Applied AI Institute\\
Deakin University\\
\And
Sunil Gupta \\
Applied AI Institute\\
Deakin University\\
\And
A.Tuan Nguyen \\
Department of Engineering Science \\
University of Oxford \\
\And
Svetha Venkatesh \\
Applied AI Institute\\
Deakin University\\
}

\iclrfinalcopy 
\begin{document}

\maketitle

\begin{abstract}
Offline policy learning (OPL) leverages existing data collected a priori for policy optimization without any active exploration. Despite the prevalence and recent interest in this problem, its theoretical and algorithmic foundations in function approximation settings remain under-developed. In this paper, we consider this problem on the axes of distributional shift, optimization, and generalization in offline contextual bandits with neural networks. In particular, we propose a provably efficient offline contextual bandit with neural network function approximation that does not require any functional assumption on the reward. We show that our method provably generalizes over unseen contexts under a milder condition for distributional shift than the existing OPL works. Notably, unlike any other OPL method, our method learns from the offline data in an online manner using stochastic gradient descent, allowing us to leverage the benefits of online learning into an offline setting. Moreover, we show that our method is more computationally efficient and has a better dependence on the effective dimension of the neural network than an online counterpart. Finally, we demonstrate the empirical effectiveness of our method in a range of synthetic and real-world OPL problems.

\end{abstract}

\section{Introduction}
We consider the problem of offline policy learning (OPL) \citep{lange2012batch,levine2020offline} where a learner infers an optimal policy given only access to a fixed dataset collected a priori by unknown behaviour policies, without any active exploration. There has been growing interest in this problem recently, as it reflects a practical paradigm where logged experiences are abundant but an interaction with the environment is often limited, with important applications in practical settings such as healthcare \citep{gottesman2019guidelines,nie2021learning}, recommendation systems \citep{strehl2010learning,thomasAAAI17}, and econometrics \citep{Kitagawa18,athey2021policy}.  

Despite the importance of OPL, theoretical and algorithmic progress on this problem has been rather limited. Specifically, most existing works are restricted to a strong parametric assumption of environments such as tabular representation \citep{DBLP:conf/aistats/YinW20,buckman2020importance,DBLP:conf/aistats/YinBW21,yin2021characterizing,rashidinejad2021bridging,xiao2021optimality} and more generally as linear models \citep{DBLP:journals/corr/abs-2002-09516,DBLP:journals/corr/abs-2012-15085,tran2021combining}. However, while the linearity assumption does not hold for many problems in practice, no work has provided a theoretical guarantee and a practical algorithm for OPL with neural network function approximation. 


In OPL with neural network function approximation, three fundamental challenges arise: 

\textbf{Distributional shift}. As OPL is provided with only a fixed dataset without any active exploration, there is often a mismatch between the distribution of the data generated by a target policy and that of the offline data. This distributional mismatch can cause erroneous value overestimation and render many standard online policy learning methods unsuccessful \citep{fujimoto2019off}. To guarantee an efficient learning under distributional shift, common analyses rely on a sort of uniform data coverage assumptions \citep{DBLP:journals/jmlr/MunosS08,DBLP:conf/icml/ChenJ19,brandfonbrener2021offline,nguyentang2021sample} that require the offline policy to be already sufficiently explorative over the entire state space and action space. 
To mitigate this strong assumption, a pessimism principle that constructs a lower confidence bound of the reward functions for decision-making \citep{rashidinejad2021bridging} can reduce the requirement to a single-policy concentration condition that requires the coverage of the offline policy only on the target policy. However, \citet{rashidinejad2021bridging} only uses this condition for tabular representation and it is unclear whether complex environments such as ones that require neural network function approximation can benefit from this condition. Moreover, the single-policy concentration condition still requires the offline policy to be stationary (e.g., the actions in the offline data are independent and depend only on current state). However, this might not hold for many practical scenarios, e.g., when the offline data was collected by an active learner (e.g., by an Q-learning algorithm). Thus, it remains unclear what is a minimal structural assumption on the distributional shift that allows a provably efficient OPL algorithm. 





\textbf{Optimization}. Solving OPL often involves in fitting a model into the offline data via optimization. Unlike in simple function approximations such as tabular representation and linear models where closed-form solutions are available, OPL with neural network function approximation poses an additional challenge that involves a non-convex, non-analytical optimization problem. 
However, existing works of OPL with function approximation ignore such optimization problems by assuming free access to an optimization oracle that can return a global minimizer \citep{brandfonbrener2021offline,duan2021risk,nguyentang2021sample} or an empirical risk minimizer with a pointwise convergence bound at an exponential rate \citep{DBLP:conf/colt/HuKU21}. This is largely not the case in practice, especially in OPL with neural network function approximation where a model is trained by gradient-based methods such as stochastic gradient descents (SGD). Thus, to understand OPL in more practical settings, it is crucial to consider optimization in design and analysis. To our knowledge, such optimization problem has not been studied in the context of OPL with neural network function approximation. 


\textbf{Generalization}. In OPL, generalization is the ability to generalize beyond the states (or contexts as in the specific case of stochastic contextual bandits) observed in the offline data. In other words, an offline policy learner with good generalization should obtain high rewards not only in the observed states but also in the entire (unknown) state distribution. 
The challenge of generalization in OPL is that as we learn from the fixed offline data, the learned policy has highly correlated structures where we cannot directly use the standard concentration inequalities (e.g. Hoeffding’s inequality, Bernstein inequality) to derive a generalization bound. The typical approaches to overcome this difficulty are data splitting and uniform convergence. While data splitting splits the offline data into disjoint folds to break the correlated structures \citep{DBLP:conf/aistats/YinW20}, uniform convergence establishes generalization uniformly over a class of policies learnable by a certain model  \citep{DBLP:conf/aistats/YinBW21}. However, in the setting where the offline data itself can have correlated structures (e.g., an offline action can depend on the previous offline data) and the model used is sufficiently large that renders a uniform convergence bound vacuous, neither the data splitting technique in \citep{DBLP:conf/aistats/YinW20} nor the uniform convergence argument in \citep{DBLP:conf/aistats/YinBW21} yield a good generalization.  
Thus, it is highly non-trivial to guarantee a strong generalization in OPL with neural network function approximation from highly correlated offline data.  

In this paper, we consider the problem
of OPL with neural network function approximation on the axes of distributional shift, optimization and generalization via studying 
the setting of stochastic contextual bandits with overparameterized neural networks. Specifically, we make three contributions toward enabling OPL in more practical settings:

\begin{itemize}
    
    \item First, we propose an algorithm that uses a neural network to model any bounded reward function without assuming any functional form (e.g., linear models) and uses a pessimistic formulation to deal with distributional shifts. Notably, unlike any standard offline learning methods, 
    our algorithm learns from the offline data in an online-like manner, allowing us to leverage the advantages of online learning into offline setting. 
    
    \item Second, our theoretical contribution lies in making the generalization bound of OPL more realistic by taking into account the optimization aspects and requiring only a milder condition for distributional shifts. In particular, our algorithm uses stochastic gradient descent and updates the network completely online, instead of retraining it from scratch for every iteration. Moreover, the distributional shift condition in our analysis, unlike in the existing works, does not require the offline policy to be uniformly explorative or stationary. Specifically, we prove that, under mild conditions with practical considerations above, our algorithm learns the optimal policy with an expected error of $ \tilde{\mathcal{O}} (\kappa  \tilde{d}^{1/2} n^{-1/2} )$, where $n$ is the number of offline samples, $\kappa$ measures the distributional shift, and $\tilde{d}$ is an effective dimension of the neural network that is much smaller than the network capacity (e.g., the network's VC dimension and Rademacher complexity).

    \item Third, we evaluate our algorithm in a number of synthetic and real-world OPL benchmark problems, verifying its empirical effectiveness against the representative methods of OPL.
\end{itemize}

\textbf{Notation}. We use lower case, bold lower case, and bold upper case to represent scalars, vectors and matrices, respectively. For a vector $\bm{v} = [v_1, \ldots, v_d]^T \in \mathbb{R}^d$ and $p > 1$, denote $\| \bm{v} \|_p = (\sum_{i=1}^d v_i^p)^{1/p}$ and let $[\bm{v}]_j$ be the $j^{\text{th}}$ element of $\bm{v}$. 
For a matrix $ \bm{A} = (A_{i,j})_{m \times n}$, denote $\| \bm{A} \|_F = \sqrt{\sum_{i,j} A_{i,j}^2}$, $\| \bm{A}\|_p = \max_{\bm{v}: \| \bm{v} \|_p = 1} \| \bm{A} \bm{v}\|_p$, $\|\bm{A}\|_{\infty} = \max_{i,j} |A_{i,j}|$ and let $\vect(\bm{A}) \in \mathbb{R}^{mn}$ be the vectorized representation of $\bm{A}$. For a square matrix $\bm{A}$, a vector $\bm{v}$, and a matrix $\bm{X}$, denote $\| \bm{v} \|_{\bm{A}} = \sqrt{ \bm{v}^T \bm{A} \bm{v}}$ and $\| \bm{X} \|_{\bm{A}} = \| \vect(\bm{X})\|_{\bm{A}}$. For a collection of matrices $\bm{W} = (\bm{W}_1, \ldots, \bm{W}_L)$ and a square matrix $\bm{A}$, denote $\| \bm{W}\|_F =\sqrt{ \sum_{l=1}^L \| \bm{W}_l \|_F^2 }$, and $ \| \bm{W} \|_{\bm{A}} = \sqrt{\sum_{l=1}^L \| \bm{W}_l\|^2_{\bm{A}}}$. For a collection of matrices $\bm{W}^{(0)} = (\bm{W}_1^{(0)}, \ldots, \bm{W}_L^{(0)})$, denote $\mathcal{B}(\bm{W}^{(0)}, R) = \{\bm{W} = (\bm{W}_1, \ldots, \bm{W}_L): \|\bm{W}_l - \bm{W}_l^{(0)}\|_F \leq R \}$. Denote $[n] = \{1,2, \ldots, n\}$, and $ a \lor b = \max\{a,b
\}$. We write $\tilde{\mathcal{O}}(\cdot)$ to hide logarithmic factors in the standard Big-Oh notation, and write $m \geq \Theta (f(\cdot))$ to indicate that there is an absolute constant $C > 0$ that is independent of any problem parameters $(\cdot)$ such that $m \geq C f(\cdot)$.

\section{Background}
In this section, we provide essential background on offline stochastic contextual bandits and overparameterized neural networks. 

\subsection{Stochastic Contextual Bandits}
We consider a stochastic $K$-armed contextual bandit where at each round $t$, an online learner observes a full context $\bm{x}_t := \{\bm{x}_{t,a} \in \mathbb{R}^d: a \in [K]\}$ sampled from a context distribution $\rho$, takes an action $a_t \in [K]$, and receives a reward $r_t \sim P(\cdot|\bm{x}_{t, a_t})$. A policy $\pi$ maps a full context (and possibly other past information) to a distribution over the action space $[K]$. For each full context $ \bm{x} :=\{\bm{x}_{a} \in \mathbb{R}^d: a \in [K]\}$, we define $v^{\pi}(\bm{x}) = \mathbb{E}_{a \sim \pi(\cdot| \bm{x}), r \sim P(\cdot|\bm{x}_{a})}[r]$ and $v^*(\bm{x}) = \max_{\pi} v^{\pi}(\bm{x})$, which is attainable due to the finite action space. 


In the offline contextual bandit setting, the goal is to learn an optimal policy only from an offline data $\mathcal{D}_n = \{ (\bm{x}_t,a_t, r_t) \}_{t=1}^n$ collected a priori by a behaviour policy $\mu$. 
The goodness of a learned policy $\hat{\pi}$ is measured by the (expected) sub-optimality the policy achieves in the entire (unknown) context distribution $\rho$:
\begin{align*}
    \subopt(\hat{\pi}) := \mathbb{E}_{\bm{x} \sim \rho} [\subopt(\hat{\pi}; \bm{x})], \text{ where } \subopt(\hat{\pi}; \bm{x}) := v^*(\bm{x}) - v^{\hat{\pi}}(\bm{x}).
\end{align*}

\noindent In this work, we make the following assumption about reward generation: For each $t$, $r_t = h( \bm{x}_{t, a_t}) + \xi_t$, 
where $h: \mathbb{R}^d \rightarrow [0,1]$ is an unknown reward function, and $\xi_t$ is a $R$-subgaussian noise conditioned on $(\mathcal{D}_{t-1}, \bm{x}_t, a_t)$ where we denote $\mathcal{D}_t = \{(\bm{x}_{\tau}, a_{\tau}, r_{\tau})\}_{1 \leq \tau \leq t}, \forall t$. 
The $R$-subgaussian noise assumption is standard in stochastic bandit literature \citep{DBLP:conf/nips/Abbasi-YadkoriPS11,zhou2020neural,xiao2021optimality} and is satisfied e.g. for any bounded noise. 

\subsection{Overparameterized Neural Networks}
To learn the unknown reward function without any prior knowledge about its parametric form, we approximate it by a neural network. In this section, we define the class of overparameterized neural networks that will be used throughout this paper. We consider fully connected neural networks with depth $L \geq 2$ defined on $\mathbb{R}^d$ as 
\begin{align}
    f_{\bm{W}}(\bm{u}) = \sqrt{m} \bm{W}_L \sigma \left( \bm{W}_{L-1} \sigma \left( \ldots \sigma( \bm{W}_1 \bm{u}  ) \ldots \right) \right), \forall \bm{u} \in \mathbb{R}^d,
    \label{eq:nn_func}
\end{align}
where $\sigma(\cdot) = \max \{ \cdot,0\}$ is the rectified linear unit (ReLU) activation function, $\bm{W}_1 \in \mathbb{R}^{m \times d}, \bm{W}_i \in \mathbb{R}^{m \times m}, \forall i \in [2,L-1], \bm{W}_L \in \mathbb{R}^{m \times 1}$, and $\bm{W} := (\bm{W}_1, \ldots, \bm{W}_L)$ with $\vect(\bm{W}) \in \mathbb{R}^p$ where $p = md + m + m^2 (L-2)$. 
We assume that the neural network is overparameterized in the sense that the width $m$ is sufficiently larger than the number of samples $n$. Under such an overparameterization regime, the dynamics of the training of the neural network can be captured in the framework of so-called neural tangent kernel (NTK) \citep{jacot2018neural}. Overparameterization has been shown to be effective to study the interpolation phenomenon and neural training for deep neural networks \citep{arora2019exact,allen2019convergence,hanin2019finite,cao2019generalization,belkin2021fit}. 



\vspace{-0.3cm}
\section{Algorithm}

\begin{algorithm}[t]
\caption{\textbf{NeuraLCB}}
\begin{algorithmic}[1]

\Require  Offline data $\mathcal{D}_n = \{(\bm{x}_t, a_t, r_t)\}_{t=1}^n$, step sizes $ \{\eta_t\}_{t=1}^n$ , regularization parameter $\lambda > 0$, confidence parameters $\{ \beta_t\}_{t=1}^n$. 

\State Initialize $\bm{W}^{(0)}$ as follows: set $\bm{W}_l^{(0)} = [\bar{\bm{W}}_l, \hspace{0.1cm} \bm{0};  \bm{0}, \hspace{0.1cm} \bar{\bm{W}}_l ], \forall l \in [L-1]$ where each entry of $\bar{\bm{W}}_l$ is generated independently from $\mathcal{N}(0, 4/m)$, and set $\bm{W}_L^{(0)} = [\bm{w}^T, -\bm{w}^T]$ where each entry of $\bm{w}$ is generated independently from $\mathcal{N}(0, 2/m)$. 

\State $\bm{\Lambda}_0 \leftarrow \lambda \bm{I}$. 

\For{$t = 1, \ldots, n$}
\State Retrieve $(\bm{x}_t, a_t, r_t)$ from $\mathcal{D}_n$. 
\State $\hat{\pi}_t(\bm{x}) \leftarrow \argmax_{a \in [K]} L_t(\bm{x}_a)$, for all $\bm{x} = \{\bm{x}_a \in \mathbb{R}^d: a \in [K]\}$ where 
$L_t(\bm{u}) = f_{\bm{W}^{(t-1)}}(\bm{u} ) - \beta_{t-1} \| \nabla f_{\bm{W}^{(t-1)}}(\bm{u} ) \cdot m^{-1/2} \|_{ \bm{\Lambda}^{-1}_{t-1} }, \forall \bm{u} \in \mathbb{R}^d$

\State $\bm{\Lambda}_t \leftarrow \bm{\Lambda}_{t-1} + \text{vec}(\nabla f_{\bm{W}^{(t-1)}}(\bm{x}_{t, a_t}) ) \cdot \text{vec}(\nabla f_{\bm{W}^{(t-1)}}(\bm{x}_{t, a_t}) )^T / m $. 

\State $\bm{W}^{(t)} \leftarrow \bm{W}^{(t-1)} - \eta_t \nabla \mathcal{L}_t(\bm{W}^{(t-1)})$ where $\mathcal{L}_t(\bm{W}) = \frac{1}{2}(f_{\bm{W}}(\bm{x}_{t,a_t}) - r_t)^2 + \frac{m \lambda}{2} \| \bm{W} - \bm{W}^{(0)} \|^2_F$. 
\EndFor

\Ensure  Randomly sample $\hat{\pi}$ uniformly from $\{\hat{\pi}_1, \ldots, \hat{\pi}_n\}$. 
\end{algorithmic}
\label{alg:neuralcb-sgd}
\end{algorithm}

In this section, we present our algorithm, namely NeuraLCB (which stands for \textbf{Neura}l \textbf{L}ower \textbf{C}onfidence \textbf{B}ound).
A key idea of NeuraLCB is to use a neural network $f_{\bm{W}}(\bm{x}_a )$ to learn the reward function $h(\bm{x}_a)$ and use a pessimism principle based on 
a lower confidence bound (LCB) of the reward function \citep{buckman2020importance,DBLP:journals/corr/abs-2012-15085} to guide decision-making. The details of NeuraLCB are presented in Algorithm \ref{alg:neuralcb-sgd}. Notably, unlike any other OPL methods, NeuraLCB learns in an online-like manner.
Specifically, at step $t$, Algorithm \ref{alg:neuralcb-sgd} retrieves $(\bm{x}_t, a_t, r_t)$ from the offline data $\mathcal{D}_n$, computes a lower confidence bound $L_t$ for each context and action based on the current network parameter $\bm{W}^{(t-1)}$, extracts a greedy policy $\hat{\pi}_t$ with respect to $L_t$, and updates $\bm{W}^{(t)}$ by minimizing a regularized squared loss function $\mathcal{L}_t(\bm{W})$ using stochastic gradient descent. Note that Algorithm \ref{alg:neuralcb-sgd} updates the network using one data point at time $t$, does not use the last sample $(\bm{x}_n, a_n, r_n)$ for decision-making and takes the average of an ensemble of policies $\{\hat{\pi}_t\}_{t=1}^n$ as its returned policy. These are merely for the convenience of theoretical analysis. In practice, we can either use the ensemble average, the best policy among the ensemble or simply the latest policy $\hat{\pi}_n$ as the returned policy. At step $t$, we can also train the network on a random batch of data from $\mathcal{D}_t$ (the ``B-mode" variant as discussed in Section \ref{section:experiment}).

\section{Generalization Analysis}
\label{section:analysis}
In this section, we analyze the generalization ability of NeuraLCB. Our analysis is built upon the neural tangent kernel (NTK) \citep{jacot2018neural}. We first define the NTK matrix for the neural network function in Eq. (\ref{eq:nn_func}).
\begin{defn}[\citet{jacot2018neural,cao2019generalization,zhou2020neural}]
Denote $\{ \bm{x}^{(i)} \}_{i=1}^{nK} = \{ \bm{x}_{t,a} \in \mathbb{R}^d : t \in [n], a \in [K] \}$, $\tilde{\bm{H}}^{(1)}_{i,j} = \bm{\Sigma}^{(1)}_{i,j} = \langle \bm{x}^{(i)}, \bm{x}^{(j)} \rangle$, and
\begin{align*}
    \bm{A}^{(l)}_{i,j} &= \begin{bmatrix} \bm{\Sigma}_{i,i}^{(l)} & \bm{\Sigma}_{i,j}^{(l)} \\ 
    \bm{\Sigma}_{i,j}^{(l)} & \bm{\Sigma}_{j,j}^{(l)}
    \end{bmatrix}, \hspace{10pt} 
    \bm{\Sigma}_{i,j}^{(l+1)} = 2 \mathbb{E}_{(u,v) \sim \mathcal{N}( \bm{0}, \bm{A}^{(l)}_{i,j})} \left[ \sigma(u) \sigma(v) \right],\\
    \tilde{\bm{H}}^{(l+1)}_{i,j} &= 2 \tilde{\bm{H}}^{(l)}_{i,j} \mathbb{E}_{(u,v) \sim \mathcal{N}( \bm{0}, \bm{A}^{(l)}_{i,j})} \left[ \sigma'(u) \sigma'(v) \right] + \bm{\Sigma}^{(l+1)}_{i,j}.
\end{align*} 
The neural tangent kernel (NTK) matrix is then defined as $\bm{H} = (\tilde{\bm{H}}^{(L)} + \bm{\Sigma}^{(L)})/2$.
\end{defn}
Here, the Gram matrix $\bm{H}$ is defined recursively from the first to the last layer of the neural network using Gaussian distributions for the observed contexts $\{ \bm{x}^{(i)} \}_{i=1}^{nK}$.
Next, we introduce the assumptions for our analysis.  First, we make an assumption about the NTK matrix $\bm{H}$ and the input data. 
\begin{assumption}
$\exists \lambda_0 > 0, \bm{H} \succeq \lambda_0 \bm{I}$, and $\forall i \in [nK]$, $ \| \bm{x}^{(i)} \|_2 = 1$. Moreover, $[\bm{x}^{(i)}]_j = [\bm{x}^{(i)}]_{j + d/2}, \forall i \in [nK], j \in [d/2]$.
\label{assumption:ntk_input_data}
\end{assumption}
The first part of Assumption \ref{assumption:ntk_input_data} assures that $\bm{H}$ is non-singular and that the input data lies in the unit sphere $\mathbb{S}^{d-1}$. Such assumption is commonly made in overparameterized neural network literature \citep{arora2019exact,DBLP:conf/iclr/DuZPS19,DBLP:conf/icml/DuLL0Z19,cao2019generalization}. The non-singularity is satisfied when e.g. any two contexts in $\{ \bm{x}^{(i)}\}$ are not parallel  \citep{zhou2020neural}, and our analysis holds regardless of whether $\lambda_0$ depends on $n$. The unit 
sphere condition is merely for the sake of analysis and can be relaxed to the case that the input data is bounded in $2$-norm. As for any input data point $\bm{x}$ such that $\| \bm{x}  \|_2 = 1$ we can always construct a new input $\bm{x}' = \frac{1}{\sqrt{2}}[\bm{x}, \bm{x}]^T$, the second part of Assumption \ref{assumption:ntk_input_data} is mild and used merely for the theoretical analysis \citep{zhou2020neural}. In particular, under Assumption \ref{assumption:ntk_input_data} and the initialization scheme in Algorithm \ref{alg:neuralcb-sgd}, we have $f_{\bm{W}^{(0)}}(\bm{x}^{(i)}) = 0, \forall i \in [nK]$.


Next, we make an assumption on the data generation.
\begin{assumption}
$\forall t, \bm{x}_t$ is independent of $\mathcal{D}_{t-1}$, and $\exists  \kappa \in (0, \infty), \left \| \frac{\pi^*(\cdot | \bm{x}_t)}{\mu(\cdot| \mathcal{D}_{t-1}, \bm{x}_t)} \right \|_{\infty} \leq \kappa, \forall t \in [n]$. 
\label{assumption:distributional_shift}
\end{assumption}
The first part of Assumption \ref{assumption:distributional_shift} says that the full contexts are generated by a process independent of any policy. This is minimal and standard in stochastic contextual bandits \citep{lattimore_szepesvari_2020,rashidinejad2021bridging,papini2021leveraging}, e.g., when $\{\bm{x}_t\}_{t=1}^n \overset{i.i.d.}{\sim} \rho$. 
The second part of Assumption \ref{assumption:distributional_shift}, namely empirical single-policy concentration (eSPC) condition, requires that the behaviour policy $\mu$ has sufficient coverage over only the optimal policy $\pi^*$ in the observed contexts. Our data coverage condition is significantly milder than the common uniform data coverage assumptions in the OPL literature \citep{DBLP:journals/jmlr/MunosS08,DBLP:conf/icml/ChenJ19,brandfonbrener2021offline,DBLP:journals/corr/abs-2012-15085,nguyentang2021sample} that requires the offline data to be sufficiently explorative in all contexts and all actions. Moreover, our data coverage condition can be considered as an extension of the single-policy concentration condition in \citep{rashidinejad2021bridging} where both require coverage over the optimal policy. However, the remarkable difference is that, unlike \citep{rashidinejad2021bridging}, the behaviour policy $\mu$ in our condition needs not to be stationary and the concentration is only defined on the observed contexts; that is, $a_t$ can be dependent on both $\bm{x}_t$ and $\mathcal{D}_{t-1}$. This is more practical as it is natural that the offline data was collected by an active learner such as a Q-learning agent \citep{mnih2015human}.

Next, we define the \textit{effective dimension} of the NTK matrix on the observed data as
    $\tilde{d} = \frac{\log\det(\bm{I} + \bm{H}/\lambda)}{\log(1 + nK/\lambda)}. $ 
    \label{eq:effective_dimension}
This notion of effective dimension was used in \citep{zhou2020neural} for online neural contextual bandits while a similar notion was introduced in \citep{valko2013finite} for online kernelized contextual bandits, and was also used in \citep{yang2020reinforcement,yang2020function} for online kernelized reinforcement learning. Although being in offline policy learning setting, the online-like nature of NeuraLCB allows us to leverage the usefulness of the effective dimension. Intuitively, $\tilde{d}$ measures how quickly the eigenvalues of $\bm{H}$ decays. For example, $\tilde{d}$ only depends on $n$ logarithmically when the eigenvalues of $\bm{H}$ have a finite spectrum (in this case $\tilde{d}$ is smaller than the number of spectrum which is the dimension of the feature space) or are exponentially decaying \citep{yang2020function}. 
We are now ready to present the main result about the sub-optimality bound of NeuraLCB.  
\begin{thm}
For any $\delta \in (0,1)$, under Assumption \ref{assumption:ntk_input_data} and \ref{assumption:distributional_shift}, if the network width $m$, the regularization parameter $\lambda$, the confidence parameters $\{\beta_t\}$ and the learning rates $\{\eta_t\}$ in Algorithm \ref{alg:neuralcb-sgd} satisfy 
\begin{align*}
    m &\geq \text{poly}(n, L, K, \lambda^{-1}, \lambda_0^{-1}, \log(1/\delta)), \hspace{0.5cm} \lambda \geq \max \{1, \Theta(L) \}, \hspace{0.5cm} \\
    \beta_t &= \sqrt{\lambda + C_3^2 tL} \cdot ( t^{1/2} \lambda^{-1/2}  + (nK)^{1/2} \lambda_0^{-1/2}  ) \cdot m^{-1/2} \text{ for some absolute constant } C_3, \\
    \eta_t &= \frac{\iota}{\sqrt{t}}, \text{ where }
    \iota^{-1} = \Omega(n^{2/3} m^{5/6} \lambda^{-1/6} L^{17/6} \log^{1/2} m) \lor \Omega(m \lambda^{1/2} \log^{1/2}(n KL^2(10n+4) / \delta)),
\end{align*}
then with probability at least $1 - \delta$ over the randomness of $\bm{W}^{(0)}$ and $\mathcal{D}_n$, the sub-optimality of $\hat{\pi}$ returned by Algorithm \ref{alg:neuralcb-sgd} is bounded as  
\begin{align*}
    n \cdot \mathbb{E} \left[\subopt(\hat{\pi}) \right] \leq  \kappa \sqrt{n}  \sqrt{\tilde{d} \log(1 + nK/\lambda)  + 2} +  \kappa \sqrt{n}  + 2 + \sqrt{2 n \log((10n+4)/\delta)},
\end{align*}
where $\tilde{d}$ is the effective dimension of the NTK matrix, and $\kappa$ is the empirical single-policy concentration (eSPC) coefficient in Assumption \ref{assumption:distributional_shift}.
\label{main_theorem}
\end{thm}
Our bound can be further simplified as 
    $\mathbb{E} [\subopt(\hat{\pi}) ] = \tilde{\mathcal{O}} (\kappa \cdot \max\{ \sqrt{\tilde{d}}, 1 \} \cdot n^{-1/2} )$. 
A detailed proof for Theorem \ref{main_theorem} is omitted to Section \ref{section:main_proof}. We make several notable remarks about our result. \emph{First}, our bound does not scale linearly with $p$ or $\sqrt{p}$ as it would if the classical analyses \citep{DBLP:conf/nips/Abbasi-YadkoriPS11,DBLP:journals/corr/abs-2012-15085} had been applied. Such a classical bound is vacuous for overparameterized neural networks where $p$ is significantly larger than $n$.  Specifically, the online-like nature of NeuraLCB allows us to leverage a matrix determinant lemma and the notion of effective dimension in online learning \citep{DBLP:conf/nips/Abbasi-YadkoriPS11,zhou2020neural} which avoids the dependence on the dimension $p$ of the feature space as in the existing OPL methods such as \citep{DBLP:journals/corr/abs-2012-15085}. \emph{Second}, as our bound scales linearly with $\sqrt{\tilde{d}}$ where $\tilde{d}$ scales only logarithmically with $n$ in common cases \citep{yang2020function}, our bound is sublinear in such cases and presents a provably efficient generalization. \emph{Third}, our bound scales linearly with $\kappa$ which does not depend on the coverage of the offline data on other actions rather than the optimal ones. This eliminates the need for a strong uniform data coverage assumption that is commonly used in the offline policy learning literature \citep{DBLP:journals/jmlr/MunosS08,DBLP:conf/icml/ChenJ19,brandfonbrener2021offline,nguyentang2021sample}. Moreover, the online-like nature of our algorithm does not necessitate the stationarity of the offline policy, allowing an offline policy with correlated structures as in many practical scenarios. Note that \citet{zhan2021off} have also recently addressed the problem of off-policy evaluation with such offline adaptive data but used doubly robust estimators instead of direct methods as in our paper. \emph{Fourth}, compared to the regret bound for online learning setting in \citep{zhou2020neural}, we achieve an improvement by a factor of $\sqrt{\tilde{d}}$ while reducing the computational complexity from $\mathcal{O}(n^2)$ to $\mathcal{O}(n)$. On a more technical note, a key idea to achieve such an improvement is to directly regress toward the optimal parameter of the neural network instead of toward the empirical risk minimizer as in \citep{zhou2020neural}. 

\emph{Finally}, to further emphasize the significance of our theoretical result, we summarize and compare it with the state-of-the-art (SOTA) sub-optimality bounds for OPL with function approximation in Table \ref{tab:compare_literature}. From the leftmost to the rightmost column, the table describes: the related works -- the function approximation -- the types of algorithms where
\emph{Pessimism} means a pessimism principle based on a lower confidence bound of the reward function while \emph{Greedy} indicates being uncertainty-agnostic (i.e., an algorithm takes an action with the highest predicted score in a given context) -- the optimization problems in OPL where \emph{Analytical} means optimization has an analytical solution, \emph{Oracle} means the algorithm relies on an oracle to obtain the global minimizer, and \emph{SGD} means the optimization is solved by stochastic gradient descent -- the sub-optimality bounds -- the data coverage assumptions where \emph{Uniform} indicates sufficiently explorative data over the context and action spaces, \emph{SPC} is the single-policy concentration condition, and \emph{eSPC} is the empirical SPC -- the nature of data generation required for the respective guarantees where \emph{I} (Independent) means that the offline actions must be sampled independently while \emph{D} (Dependent) indicates that the offline actions can be dependent on the past data. It can be seen that our result has a stronger generalization under the most practical settings as compared to the existing SOTA generalization theory for OPL. {We also remark that the optimization design and guarantee in NeuraLCB (single data point SGD) are of independent interest that do not only apply to the offline setting but also to the original online setting in \citep{zhou2020neural} to improve their regret and optimization complexity.}

\begin{table}
    \centering
    \caption{The SOTA generalization theory of OPL with function approximation.}
     \resizebox{\textwidth}{!}{  
    \begin{tabular}{c|c|c|c|c|c|c|}
    \hline
       \textbf{Work} &  \textbf{Function}  & \textbf{Type} & \textbf{Optimization} & \textbf{Sub-optimality}  & \textbf{Data Coverage} & \textbf{Data Gen.}\\
    \hline
    \hline 
    
    \citet{DBLP:conf/aistats/YinW20}$^a$ & Tabular & Greedy & Analytical & $\tilde{\mathcal{O}}\left( \sqrt{ | \mathcal{X}| \cdot K} \cdot n^{-1/2}\right)$ & Uniform & I \\ 
    
    \citet{rashidinejad2021bridging} & Tabular & Pessimism & Analytical & $\tilde{\mathcal{O}}\left( \sqrt{|\mathcal{X}| \cdot \kappa} \cdot n^{-1/2} \right)$ & SPC & I  \\ 
    
     \citet{DBLP:journals/corr/abs-2002-09516}$^b$ & Linear & Greedy  & Analytical & $\tilde{\mathcal{O}}\left( \kappa \cdot n^{-1/2} + d \cdot n^{-1} \right)$ & Uniform  & I \\ 
    
    \citet{DBLP:journals/corr/abs-2012-15085} & Linear & Pessimism & Analytical & $\tilde{\mathcal{O}}\left( d \cdot n^{-1/2} \right)$ & Uniform & I\\ 
    
    \citet{nguyentang2021sample} &  Narrow ReLU  & Greedy & Oracle & $\tilde{\mathcal{O}}\left(  \sqrt{\kappa} \cdot n^{- \frac{\alpha}{2(\alpha + d)}} \right)$ & Uniform & I \\
    
    \textbf{This work} & Wide ReLU & Pessimism & SGD & $\tilde{\mathcal{O}}(\kappa \cdot \sqrt{\tilde{d}} \cdot n^{-1/2})$ & eSPC & I/D \\
    \hline
    \end{tabular}
    }
    \flushleft
    \footnotesize{$^{a,b}$ The bounds of these works are for off-policy evaluation which is generally easier than OPL problem.}
    \label{tab:compare_literature}
\end{table}

\nopagebreak 
\section{Related Work}
\label{section:background_and_related_literature}

\textbf{OPL with function approximation}. Most OPL works, in both bandit and reinforcement learning settings, use tabular representation \citep{DBLP:conf/aistats/YinW20,buckman2020importance,DBLP:conf/aistats/YinBW21,yin2021characterizing,rashidinejad2021bridging,xiao2021optimality} and linear models \citep{DBLP:journals/corr/abs-2002-09516,DBLP:journals/corr/abs-2012-15085,tran2021combining}. The most related work on OPL with neural function approximation we are aware of are \citep{brandfonbrener2021offline,nguyentang2021sample,uehara2021finite}. However, \citet{brandfonbrener2021offline,nguyentang2021sample,uehara2021finite} rely on a strong uniform data coverage assumption and an optimization oracle to the empirical risk minimizer. Other OPL analyses with general function approximation \citep{duan2021risk,DBLP:conf/colt/HuKU21} also use such optimization oracle, which limit the applicability of their algorithms and analyses in practical settings.

To deal with nonlinear rewards without making strong functional assumptions, other approaches rather than neural networks have been considered, including a family of experts \citep{10.5555/944919.944941}, a reduction to supervised learning \citep{langford2007epoch,agarwal2014taming}, and nonparametric models \citep{kleinberg2008multi,srinivas2009gaussian,krause2011contextual,bubeck2011x,valko2013finite}. However, they are all in the online policy learning setting instead of the OPL setting. Moreover, these approaches have time complexity scaled linearly with the number of experts, rely the regret on an oracle, and have cubic computational complexity, respectively, while our method with neural networks are both statically and computationally efficient where it achieves a $\sqrt{n}$-type suboptimality bound and only linear computational complexity.

\textbf{Neural networks}. Our work is inspired by the theoretical advances of neural networks and their subsequent application in online policy learning \citep{yang2020function,zhou2020neural,xu2020neural}. For optimization aspect, (stochastic) gradient descents can provably find global minima of training loss of neural networks \citep{DBLP:conf/icml/DuLL0Z19,DBLP:conf/iclr/DuZPS19,allen2019convergence,nguyen2021proof}. For generalization aspect, (stochastic) gradient descents can train an overparameterized neural network to a regime where the neural network interpolates the training data (i.e., zero training error) and has a good generalization ability \citep{arora2019exact,cao2019generalization,belkin2021fit}. 

Regarding the use of neural networks for policy learning, NeuraLCB is similar to the NeuralUCB algorithm \citep{zhou2020neural}, that is proposed for the setting of online contextual bandits with neural networks, in the sense that both algorithms use neural networks, learn with a streaming data and construct a (lower and upper, respectively) confidence bound of the reward function to guide decision-making. Besides the apparent difference of offline and online policy learning problems, the notable difference of NeuraLCB from NeuralUCB is that while NeuralUCB trains a new neural network from scratch at each iteration (for multiple epochs), NeuraLCB trains a single neural network completely online. That is, NeuraLCB updates the neural network in light of the data at a current iteration from the trained network of the previous iteration. Such optimization scheme in NeuraLCB greatly reduces the computational complexity from $\mathcal{O}(n^2)$ \footnote{In practice, at each time step $t$, NeuralUCB trains a neural network for $t$ epochs using gradient descent in the entire data collected up to time $t$.} to $O(n)$, while still guaranteeing a provably efficient algorithm. Moreover, NeuraLCB achieves a suboptimality bound with a better dependence on the effective dimension than NeuralUCB.  

\section{Experiments}
\label{section:experiment}

In this section, we evaluate NeuraLCB and compare it with five representative baselines: (1) LinLCB \citep{DBLP:journals/corr/abs-2012-15085}, which also uses LCB but relies on linear models, (2) KernLCB, which approximates the reward using functions in a RKHS and is an offline counterpart of KernelUCB \citep{valko2013finite}, (3) NeuralLinLCB, which is the same as LinLCB except that it uses $\phi(\bm{x}_a) = \vect(\nabla f_{\bm{W}^{(0)}}(\bm{x}_a)) $ as the feature extractor for the linear model where $f_{\bm{W}^{(0)}}$ is the same neural network of NeuraLCB at initialization, (4) NeuralLinGreedy, which is the same as NeuralLinLCB except that it relies on the empirical estimate of the reward function for decision-making, and (5) NeuralGreedy, which is the same as NeuraLCB except that NeuralGreedy makes decision based on the empirical estimate of the reward. For more details and completeness, we present the pseudo-code of these baseline methods in Section \ref{section:algo_baseline}. Here, we compare the algorithms by their generalization ability via their expected sub-optimality. For each algorithm, we vary the number of (offline) samples $n$ from $1$ to $T$ where $T$ will be specified in each dataset, and repeat each experiment for $10$ times. We report the mean results and their $95\%$ confidence intervals. \footnote{Our code repos: \url{https://github.com/thanhnguyentang/offline_neural_bandits}.} 

\begin{figure}
\centering
\begin{minipage}[p]{0.32\linewidth}
\centering
\includegraphics[width=\linewidth]{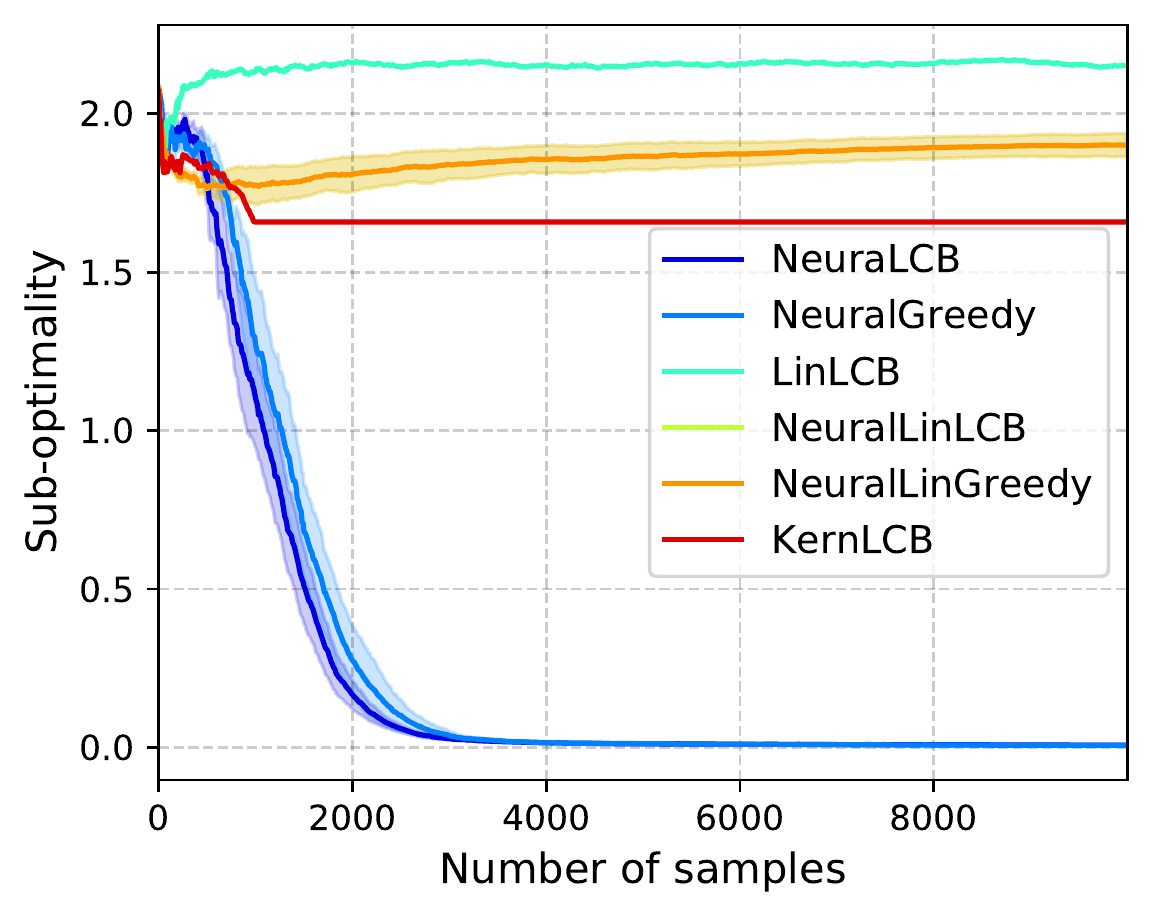}
\subcaption{$h(\bm{u}) = 10 (\bm{a}^T u)^2$} 
\label{fig:first}
\end{minipage}
\begin{minipage}[p]{0.32\linewidth}
\centering
\includegraphics[width=\linewidth]{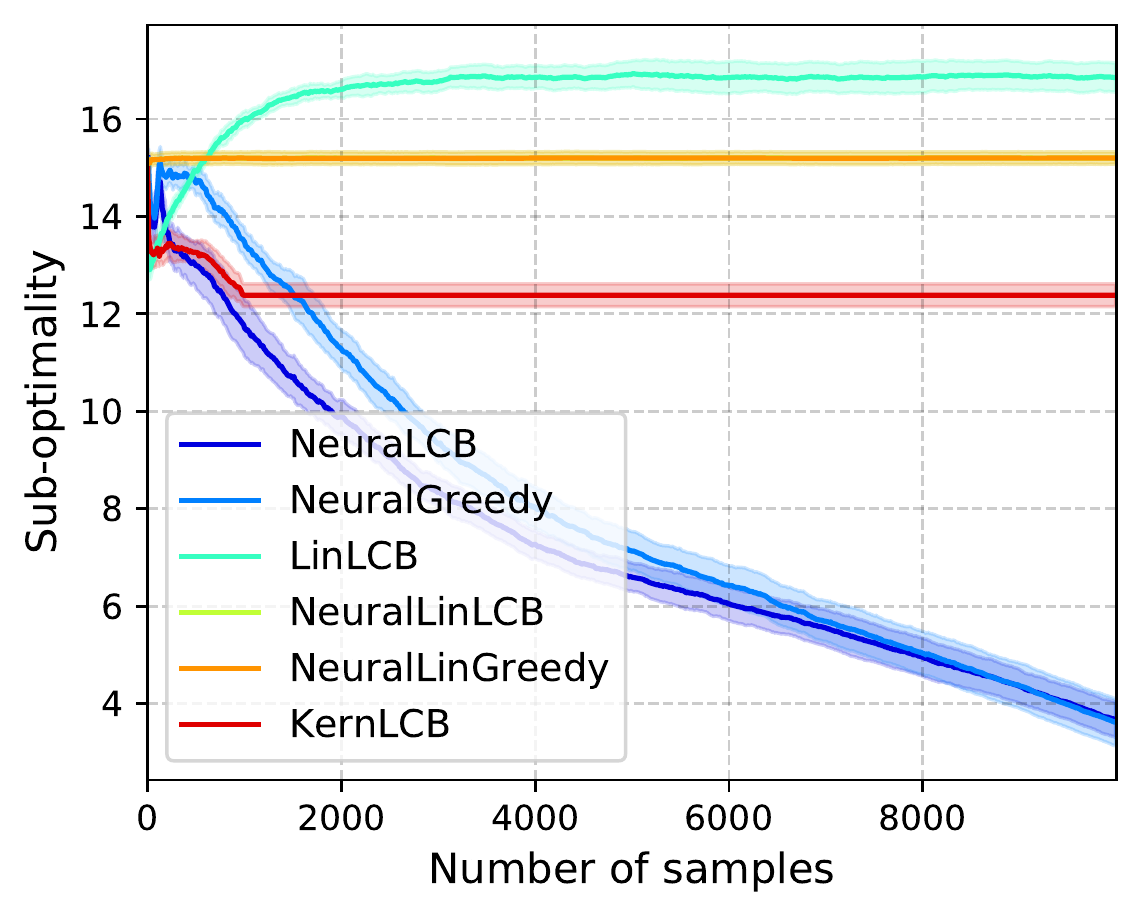}
\subcaption{$h(\bm{u}) = \bm{u}^T \bm{A}^{T} \bm{A} \bm{u}$}
\end{minipage}
\begin{minipage}[p]{0.32\linewidth}
\centering
\includegraphics[width=\linewidth]{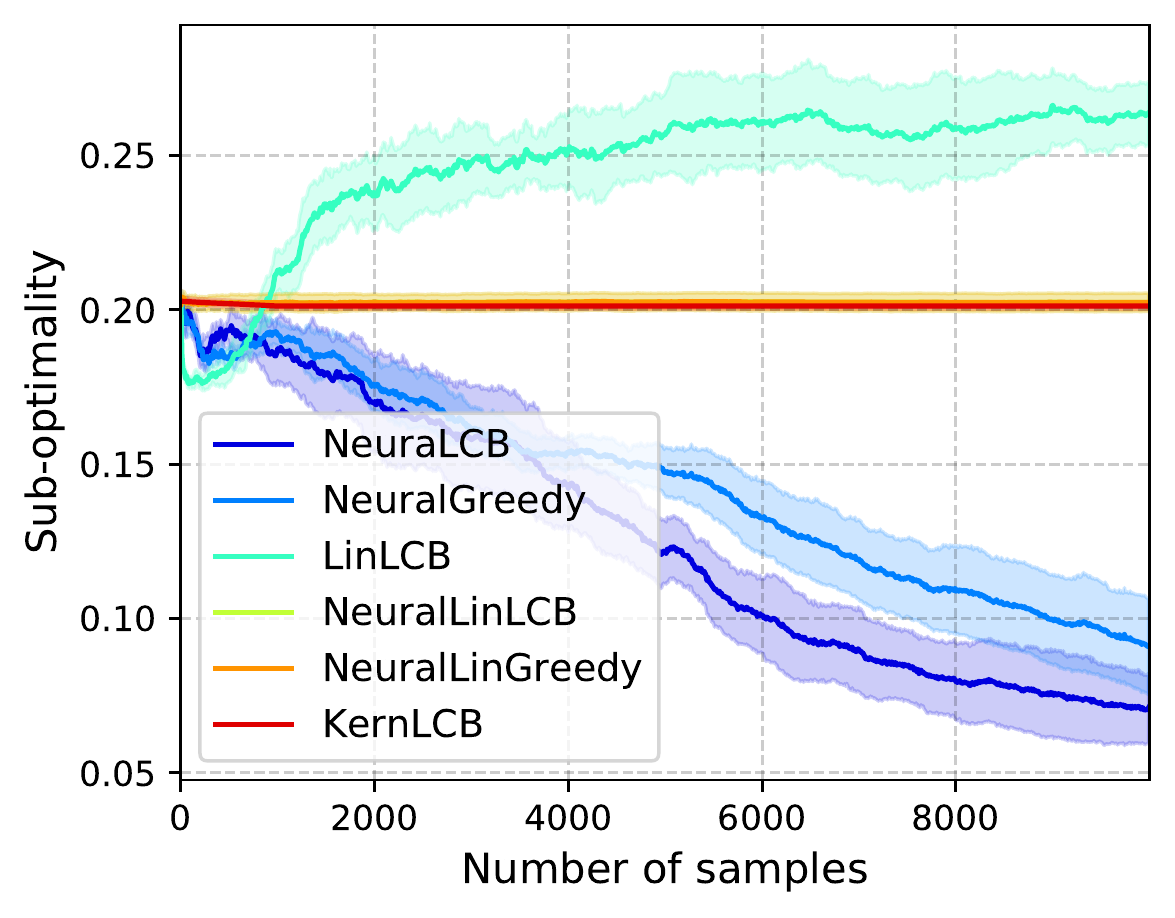}
\subcaption{$h(\bm{u}) = cos(3 \bm{a}^T \bm{u})$}
\end{minipage}
\caption{The sub-optimality of NeuraLCB versus the baseline algorithms on synthetic datasets.}
\label{fig:syn}
\end{figure}

\textbf{Approximation}. To accelerate computation, we follow \citep{riquelme2018deep,zhou2020neural} to approximate large covariance matrices and expensive kernel methods. Specifically, as NeuraLCB and NeuralLinLCB involve computing a covariance matrix $\bm{\Lambda}_t$ of size $p \times p$ where $p = md + m + m^2 (L-2)$ is the number of the neural network parameters which could be large, we approximate  $\bm{\Lambda}_t$ by its diagonal. Moreover, as KernLCB scales cubically with the number of samples, we use KernLCB fitted on the first $1,000$ samples if the offline data exceeds $1,000$ samples. 

\textbf{Data generation}. We generate offline actions using a fixed $\epsilon$-greedy policy with respect to the true reward function of each considered contextual bandit, where $\epsilon$ is set to $0.1$ for all experiments in this section. In each run, we randomly sample $n_{te}=10,000$ contexts from $\rho$ and use this same test contexts to approximate the expected sub-optimality of each algorithm. 

\textbf{Hyperparameters}. 
We fix $\lambda = 0.1$ for all algorithms. For NeuraLCB, we set $\beta_t = \beta$, and for NeuraLCB, LinLCB, KernLCB, and NeuralLinLCB, we do grid search over $\{0.01, 0.05, 0.1, 1, 5, 10\}$ for the uncertainty parameter $\beta$. For KernLCB, we use the radius basis function (RBF) kernel with parameter $\sigma$ and do grid search over $\{0.1, 1, 10\}$ for $\sigma$. For NeuraLCB and NeuralGreedy, we use Adam optimizer \citep{kingma2014adam} with learning rate $\eta$ grid-searched over $\{0.0001, 0.001\}$ and set the $l_2$-regularized parameter to $0.0001$. For NeuraLCB, for each $\mathcal{D}_t$, we use $\hat{\pi}_t$ as its final returned policy instead of averaging over all policies $\{ \hat{\pi}_{\tau}\}_{\tau=1}^t$. Moreover, we grid search NeuraLCB and NeuralGreedy over two training modes, namely $\{\text{S-mode}, \text{B-mode}\}$ where at each iteration $t$, S-mode updates the neural network for one step of SGD (one step of Adam update in practice) on one single data point $(\bm{x}_t, a_t, r_t)$ while B-mode updates the network for $100$ steps of SGD on a random batch of size $50$ of data $\mathcal{D}_t$ (details at Algorithm \ref{alg:neuralcb-sgd-bmode}). We remark that even in the B-mode, NeuraLCB is still more computationally efficient than its online counterpart NeuralUCB \citep{zhou2020neural} as NeuraLCB reuses the neural network parameters from the previous iteration instead of training it from scratch for each new iteration. For NeuralLinLCB, NeuralLinGreedy, NeuraLCB, and NeuralGreedy, we use the same network architecture with $L = 2$ and add layer normalization \citep{ba2016layer} in the hidden layers. The network width $m$ will be specified later based on datasets. 

\subsection{Synthetic Datasets}
For synthetic experiments, we evaluate the algorithms on contextual bandits with the synthetic nonlinear reward functions $h$ used in \citep{zhou2020neural}: 
\begin{align*}
    h_1(\bm{u}) = 10(\bm{u}^T \bm{a})^2, \hspace{8pt} h_2(\bm{u}) &= \bm{u}^T \bm{A}^T \bm{A} \bm{u}, \hspace{8pt} h_3(\bm{u}) = \cos(3 \bm{u}^T \bm{a}),
\end{align*}
where $\bm{a} \in \mathbb{R}^d$ is randomly generated from uniform distribution over the unit sphere, and each entry of $\bm{A} \in \mathbb{R}^{d \times d}$ is independently and randomly generated from $\mathcal{N}(0,1)$. For each reward function $h_i$, $r_t = h_i(\bm{x}_{t,a_t}) + \xi_t$ where $\xi_t \sim \mathcal{N}(0, 0.1)$. The context distribution $\rho$ for three cases is the uniform distribution over the unit sphere. All contextual bandit instances have context dimension $d = 20$ and $K = 30$ actions. Moreover, we choose the network width $m = 20$ and the maximum number of samples $T = 10,000$ for the synthetic datasets.
\subsection{Real-world Datasets}
We evaluate the algorithms on real-world datasets from UCI Machine Learning Repository \citep{UCIData}: \emph{Mushroom}, \emph{Statlog}, and \emph{Adult}, and {\emph{MNIST} \citep{lecun1998gradient}}. They represent a good range of properties: small versus large sizes, dominating actions, and stochastic versus deterministic rewards (see Section \ref{sec:datasets} in the appendix for details on each dataset). 
Besides the \emph{Mushroom} bandit, \emph{Statlog}, \emph{Adult}, and \emph{MNIST} are $K$-class classification datasets, which we convert into $K$-armed contextual bandit problems, following \citep{riquelme2018deep,zhou2020neural}. Specifically, for each input $\bm{x} \in \mathbb{R}^d$ in a $K$-class classification problem, we create $K$ contextual vectors $\bm{x}^1 = (\bm{x}, \bm{0}, \ldots, \bm{0}), \ldots, \bm{x}^{K} = (\bm{0}, \ldots, \bm{0}, \bm{x}) \in \mathbb{R}^{dK}$. The learner receives reward $1$ if it selects context $\bm{x}^y$ where $y$ is the label of $\bm{x}$, and receives reward $0$ otherwise. Moreover, we choose the network width $m = 100$ and the maximum sample number $T = 15,000$ for these datasets. 

\subsection{Results}
Figure \ref{fig:syn} and \ref{fig:realworld} show the expected sub-optimality of all algorithms on synthetic datasets and real-world datasets, respectively. First, due to the non-linearity of the reward functions, methods with linear models (LinLCB, NeuralLinLCB, NeuralLinGreedy, and KernLCB) fail in almost all tasks (except that LinLCB and KernLCB have competitive performance in \emph{Mushroom} and \emph{Adult}, respectively). In particular, linear models using neural network features without training (NeuralLinLCB and NeuralLinGreedy) barely work in any datasets considered here. In contrast, our method NeuraLCB outperforms all the baseline methods in all tasks. We remark that NeuralGreedy has a substantially lower sub-optimality in synthetic datasets (and even a comparable performance with NeuraLCB in $h_1$) than linear models, suggesting the importance of using trainable neural representation in highly non-linear rewards instead of using linear models with fixed feature extractors. On the other hand, our method NeuraLCB outperforms NeuralGreedy in real-world datasets by a large margin (even though two methods are trained exactly the same but different only in decision-making), confirming the effectiveness of pessimism principle in these tasks. Second, KernLCB performs reasonably in certain OPL tasks (\emph{Adult} and slightly well in \emph{Statlog} and \emph{MNIST}), but the cubic computational complexity of kernel methods make it less appealing in OPL with offline data at more than moderate sizes. Note that due to such cubic computational complexity, we follow \citep{riquelme2018deep,zhou2020neural} to learn the kernel in KernLCB for only the first $1,000$ samples and keep the fitted kernel for the rest of the data (which explains the straight line of KernLCB after $n=1,000$ in our experiments). Our method NeuraLCB, on the other hand, is highly computationally efficient as the computation scales only linearly with $n$ (even in the B-mode -- a slight modification of the original algorithm to include batch training). In fact, in the real-world datasets above, the S-mode (which trains in a single data point for one SGD step at each iteration) outperforms the B-mode, further confirming the effectiveness of the online-like nature in NeuraLCB. {In Section \ref{sec:add_expr}, we reported the performance of S-mode and B-mode together and evaluated on dependent offline data.




\begin{figure}
\centering
\begin{minipage}[p]{0.24\linewidth}
\centering
\includegraphics[width=\linewidth]{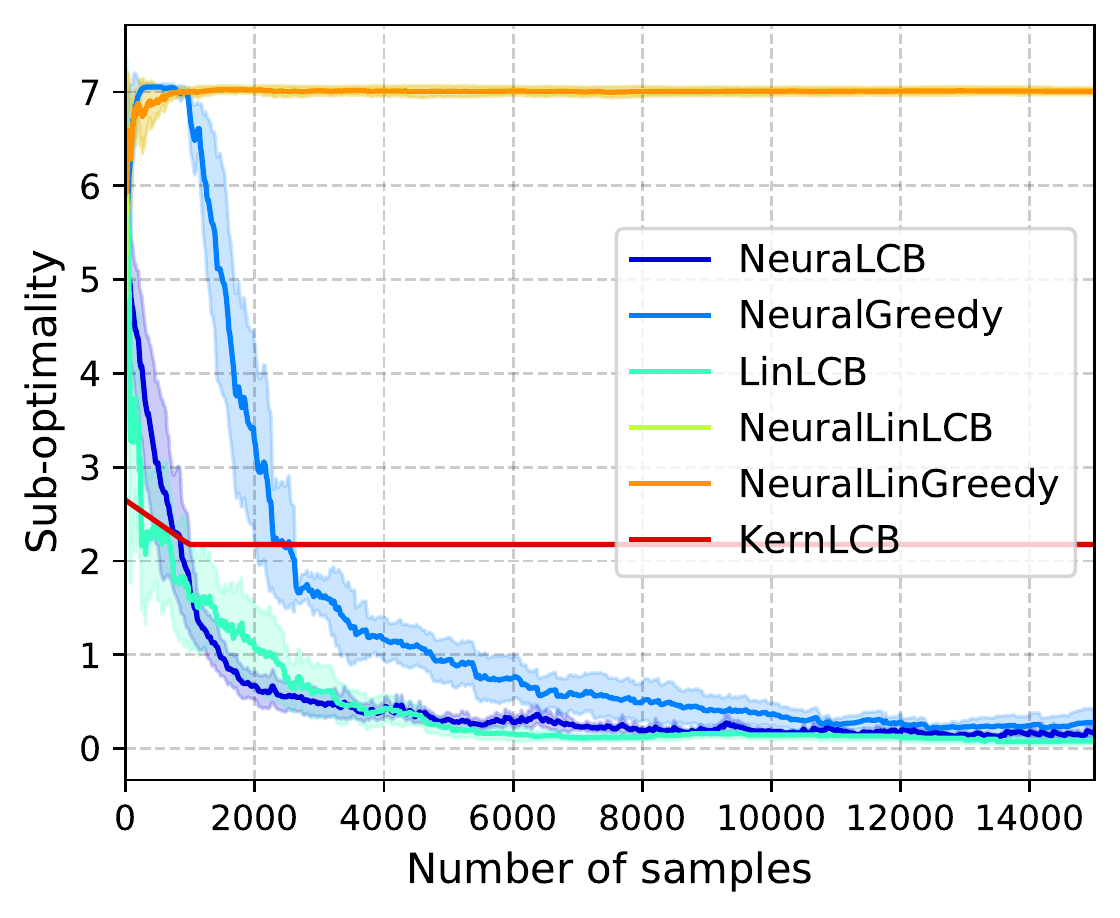}
\subcaption{Mushroom} 
\label{fig:first}
\end{minipage}
\begin{minipage}[p]{0.24\linewidth}
\centering
\includegraphics[width=\linewidth]{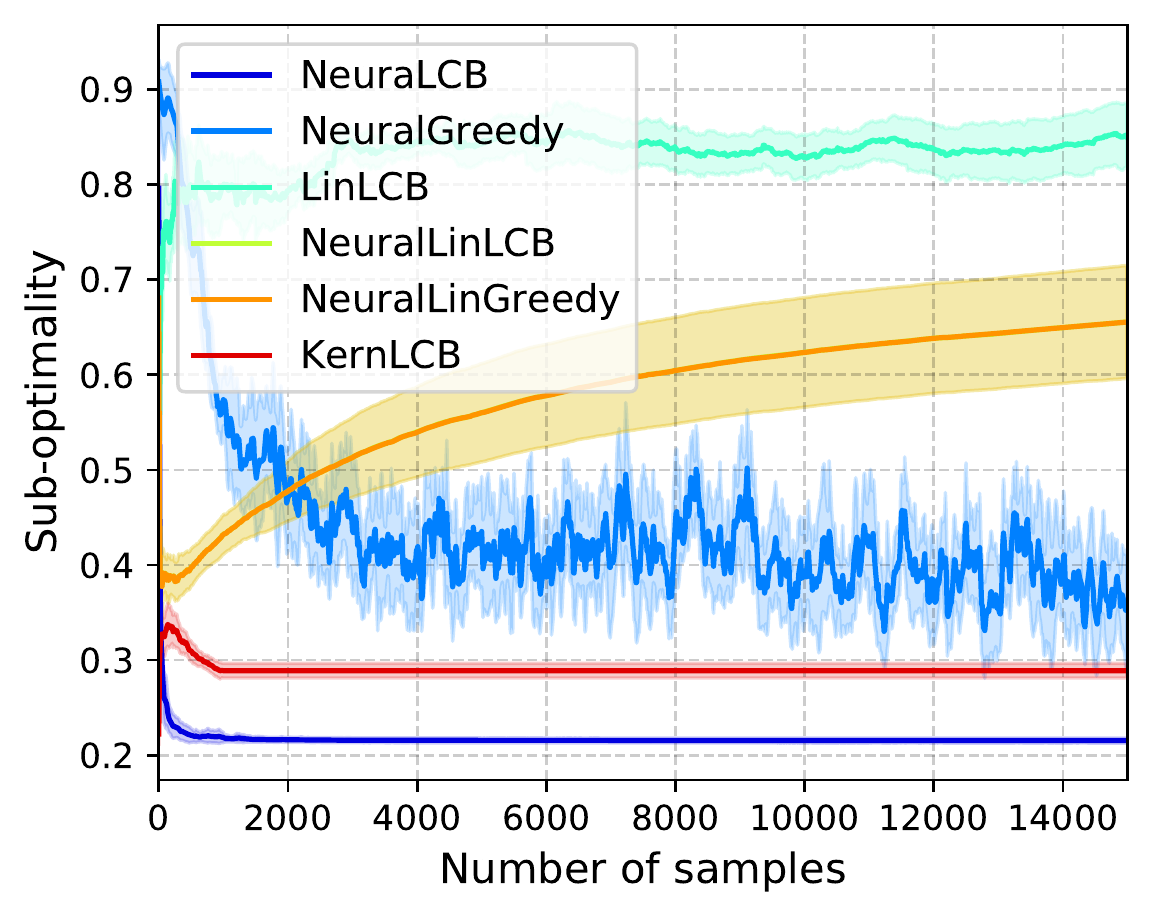}
\subcaption{Statlog}
\end{minipage}
\begin{minipage}[p]{0.24\linewidth}
\centering
\includegraphics[width=\linewidth]{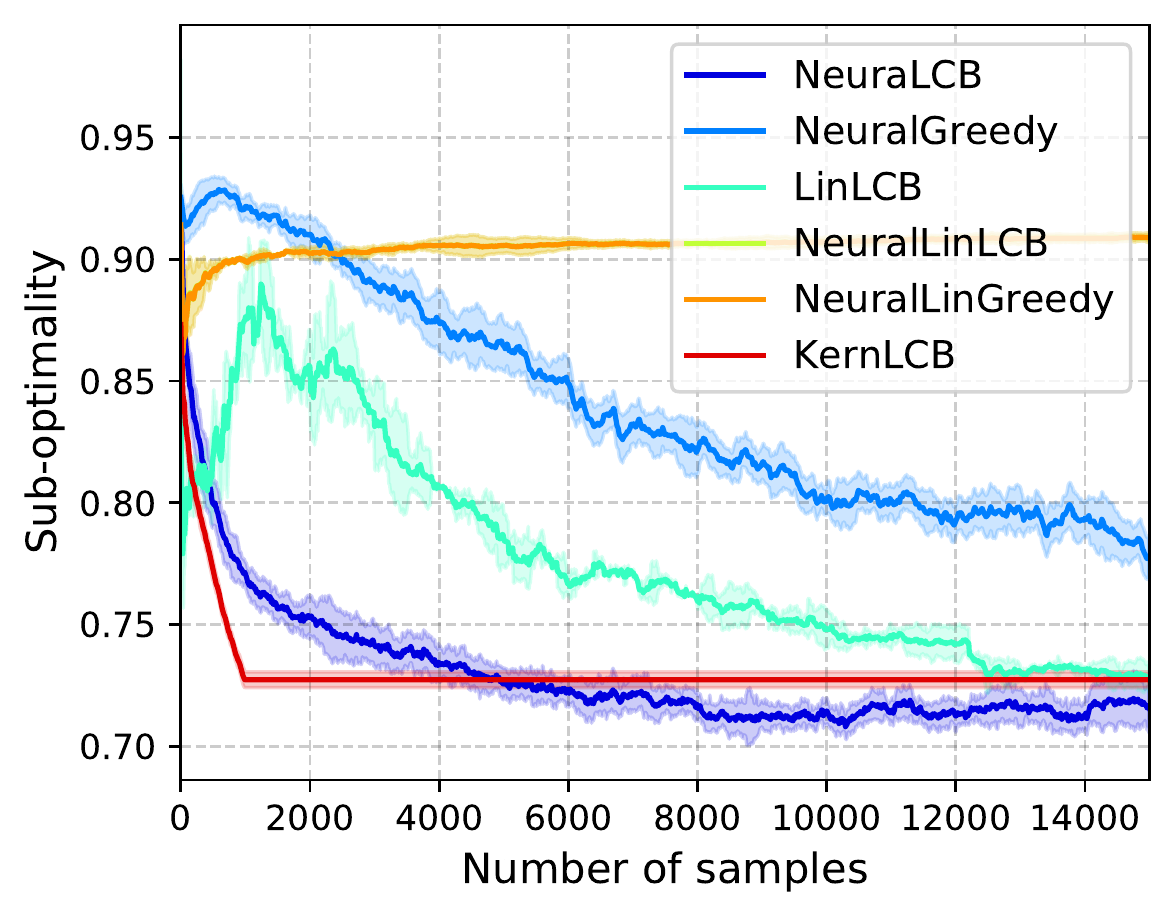}
\subcaption{Adult}
\end{minipage}
\begin{minipage}[p]{0.24\linewidth}
\centering
\includegraphics[width=\linewidth]{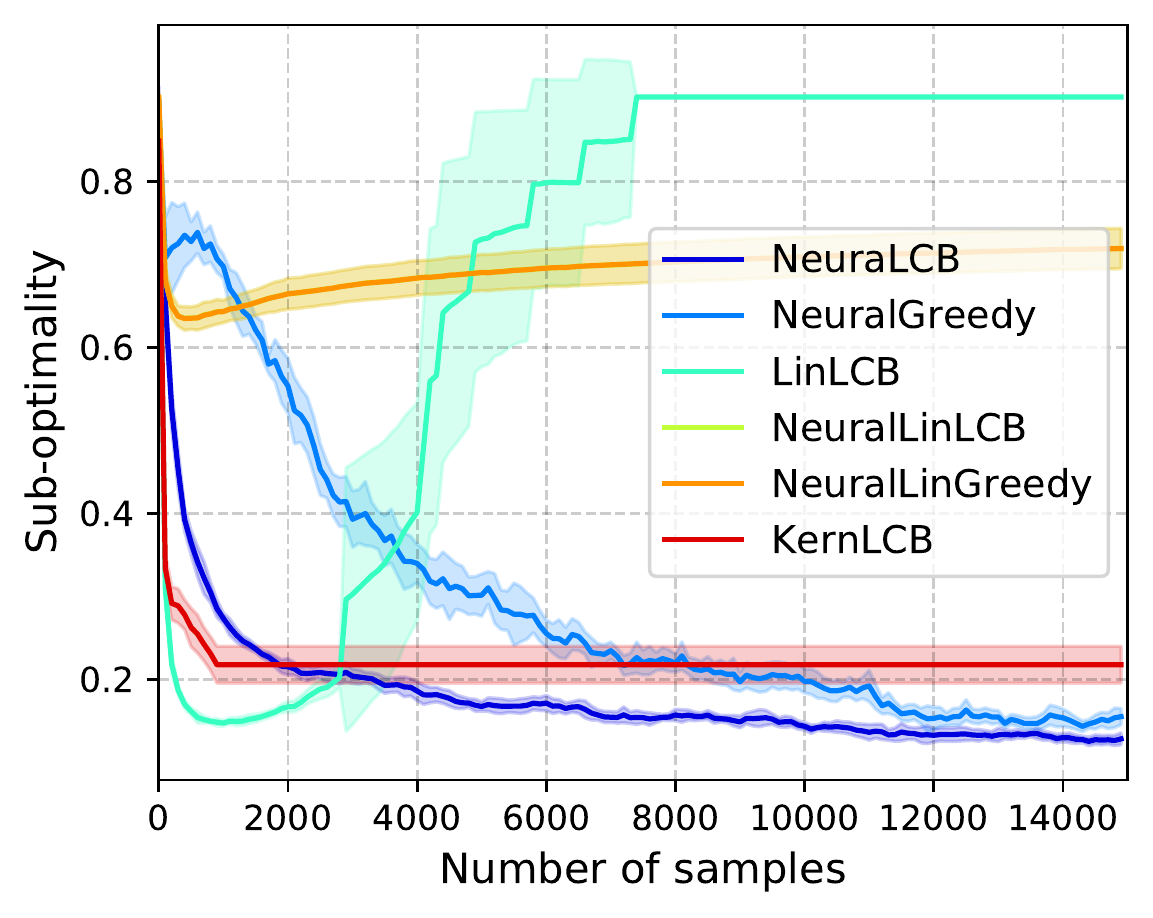}
\subcaption{MNIST}
\end{minipage}
\caption{The sub-optimality of NeuraLCB versus the baseline algorithms on real-world datasets.}
\label{fig:realworld}
\end{figure}
\section{Conclusion}
In this paper, we proposed NeuraLCB, a new algorithm for offline contextual bandits with neural network function approximation and pessimism principle. We proved that our algorithm achieves a sub-optimality bound of $\tilde{\mathcal{O}} (\kappa \cdot  \sqrt{\tilde{d}} \cdot n^{-1/2} )$ under our new data coverage condition that is milder than the existing ones. This is also the first provably efficient result of OPL with neural network function approximation that considers practical optimization. The promising empirical results in both synthetic and real-world datasets suggested the effectiveness of our algorithm in practice.


\section{Acknowledgement}
This research was partially funded by the Australian Government through the Australian Research Council (ARC). Prof. Venkatesh is the recipient of an ARC Australian Laureate Fellowship (FL170100006).
\bibliographystyle{iclr2022_conference}
\bibliography{main}

\begin{thebibliography}{57}
\providecommand{\natexlab}[1]{#1}
\providecommand{\url}[1]{\texttt{#1}}
\expandafter\ifx\csname urlstyle\endcsname\relax
  \providecommand{\doi}[1]{doi: #1}\else
  \providecommand{\doi}{doi: \begingroup \urlstyle{rm}\Url}\fi

\bibitem[Abbasi{-}Yadkori et~al.(2011)Abbasi{-}Yadkori, P{\'{a}}l, and
  Szepesv{\'{a}}ri]{DBLP:conf/nips/Abbasi-YadkoriPS11}
Yasin Abbasi{-}Yadkori, D{\'{a}}vid P{\'{a}}l, and Csaba Szepesv{\'{a}}ri.
\newblock Improved algorithms for linear stochastic bandits.
\newblock In John Shawe{-}Taylor, Richard~S. Zemel, Peter~L. Bartlett, Fernando
  C.~N. Pereira, and Kilian~Q. Weinberger (eds.), \emph{Advances in Neural
  Information Processing Systems 24: 25th Annual Conference on Neural
  Information Processing Systems 2011. Proceedings of a meeting held 12-14
  December 2011, Granada, Spain}, pp.\  2312--2320, 2011.
\newblock URL
  \url{https://proceedings.neurips.cc/paper/2011/hash/e1d5be1c7f2f456670de3d53c7b54f4a-Abstract.html}.

\bibitem[Agarwal et~al.(2014)Agarwal, Hsu, Kale, Langford, Li, and
  Schapire]{agarwal2014taming}
Alekh Agarwal, Daniel Hsu, Satyen Kale, John Langford, Lihong Li, and Robert
  Schapire.
\newblock Taming the monster: A fast and simple algorithm for contextual
  bandits.
\newblock In \emph{International Conference on Machine Learning}, pp.\
  1638--1646. PMLR, 2014.

\bibitem[Allen-Zhu et~al.(2019)Allen-Zhu, Li, and Song]{allen2019convergence}
Zeyuan Allen-Zhu, Yuanzhi Li, and Zhao Song.
\newblock A convergence theory for deep learning via over-parameterization.
\newblock In \emph{International Conference on Machine Learning}, pp.\
  242--252. PMLR, 2019.

\bibitem[Arora et~al.(2019)Arora, Du, Hu, Li, Salakhutdinov, and
  Wang]{arora2019exact}
Sanjeev Arora, Simon~S Du, Wei Hu, Zhiyuan Li, Ruslan Salakhutdinov, and
  Ruosong Wang.
\newblock On exact computation with an infinitely wide neural net.
\newblock \emph{arXiv preprint arXiv:1904.11955}, 2019.

\bibitem[Athey \& Wager(2021)Athey and Wager]{athey2021policy}
Susan Athey and Stefan Wager.
\newblock Policy learning with observational data.
\newblock \emph{Econometrica}, 89\penalty0 (1):\penalty0 133--161, 2021.

\bibitem[Auer(2003)]{10.5555/944919.944941}
Peter Auer.
\newblock Using confidence bounds for exploitation-exploration trade-offs.
\newblock \emph{J. Mach. Learn. Res.}, 3\penalty0 (null):\penalty0 397–422,
  March 2003.
\newblock ISSN 1532-4435.

\bibitem[Ba et~al.(2016)Ba, Kiros, and Hinton]{ba2016layer}
Jimmy~Lei Ba, Jamie~Ryan Kiros, and Geoffrey~E Hinton.
\newblock Layer normalization.
\newblock \emph{arXiv preprint arXiv:1607.06450}, 2016.

\bibitem[Belkin(2021)]{belkin2021fit}
Mikhail Belkin.
\newblock Fit without fear: remarkable mathematical phenomena of deep learning
  through the prism of interpolation.
\newblock \emph{arXiv preprint arXiv:2105.14368}, 2021.

\bibitem[Brandfonbrener et~al.(2021)Brandfonbrener, Whitney, Ranganath, and
  Bruna]{brandfonbrener2021offline}
David Brandfonbrener, William Whitney, Rajesh Ranganath, and Joan Bruna.
\newblock Offline contextual bandits with overparameterized models.
\newblock In \emph{International Conference on Machine Learning}, pp.\
  1049--1058. PMLR, 2021.

\bibitem[Bubeck et~al.(2011)Bubeck, Munos, Stoltz, and
  Szepesv{\'a}ri]{bubeck2011x}
S{\'e}bastien Bubeck, R{\'e}mi Munos, Gilles Stoltz, and Csaba Szepesv{\'a}ri.
\newblock X-armed bandits.
\newblock \emph{Journal of Machine Learning Research}, 12\penalty0 (5), 2011.

\bibitem[Buckman et~al.(2020)Buckman, Gelada, and
  Bellemare]{buckman2020importance}
Jacob Buckman, Carles Gelada, and Marc~G Bellemare.
\newblock The importance of pessimism in fixed-dataset policy optimization.
\newblock \emph{arXiv preprint arXiv:2009.06799}, 2020.

\bibitem[Cao \& Gu(2019)Cao and Gu]{cao2019generalization}
Yuan Cao and Quanquan Gu.
\newblock Generalization bounds of stochastic gradient descent for wide and
  deep neural networks.
\newblock \emph{Advances in Neural Information Processing Systems},
  32:\penalty0 10836--10846, 2019.

\bibitem[Cesa{-}Bianchi et~al.(2004)Cesa{-}Bianchi, Conconi, and
  Gentile]{DBLP:journals/tit/Cesa-BianchiCG04}
Nicol{\`{o}} Cesa{-}Bianchi, Alex Conconi, and Claudio Gentile.
\newblock On the generalization ability of on-line learning algorithms.
\newblock \emph{{IEEE} Trans. Inf. Theory}, 50\penalty0 (9):\penalty0
  2050--2057, 2004.
\newblock \doi{10.1109/TIT.2004.833339}.
\newblock URL \url{https://doi.org/10.1109/TIT.2004.833339}.

\bibitem[Chen \& Jiang(2019)Chen and Jiang]{DBLP:conf/icml/ChenJ19}
Jinglin Chen and Nan Jiang.
\newblock Information-theoretic considerations in batch reinforcement learning.
\newblock In \emph{{ICML}}, volume~97 of \emph{Proceedings of Machine Learning
  Research}, pp.\  1042--1051. {PMLR}, 2019.

\bibitem[Du et~al.(2019{\natexlab{a}})Du, Lee, Li, Wang, and
  Zhai]{DBLP:conf/icml/DuLL0Z19}
Simon~S. Du, Jason~D. Lee, Haochuan Li, Liwei Wang, and Xiyu Zhai.
\newblock Gradient descent finds global minima of deep neural networks.
\newblock In \emph{{ICML}}, volume~97 of \emph{Proceedings of Machine Learning
  Research}, pp.\  1675--1685. {PMLR}, 2019{\natexlab{a}}.

\bibitem[Du et~al.(2019{\natexlab{b}})Du, Zhai, P{\'{o}}czos, and
  Singh]{DBLP:conf/iclr/DuZPS19}
Simon~S. Du, Xiyu Zhai, Barnab{\'{a}}s P{\'{o}}czos, and Aarti Singh.
\newblock Gradient descent provably optimizes over-parameterized neural
  networks.
\newblock In \emph{{ICLR} (Poster)}. OpenReview.net, 2019{\natexlab{b}}.

\bibitem[Dua \& Graff(2017)Dua and Graff]{UCIData}
Dheeru Dua and Casey Graff.
\newblock {UCI} machine learning repository, 2017.
\newblock URL \url{http://archive.ics.uci.edu/ml}.

\bibitem[Duan \& Wang(2020)Duan and Wang]{DBLP:journals/corr/abs-2002-09516}
Yaqi Duan and Mengdi Wang.
\newblock Minimax-optimal off-policy evaluation with linear function
  approximation.
\newblock \emph{CoRR}, abs/2002.09516, 2020.

\bibitem[Duan et~al.(2021)Duan, Jin, and Li]{duan2021risk}
Yaqi Duan, Chi Jin, and Zhiyuan Li.
\newblock Risk bounds and rademacher complexity in batch reinforcement
  learning.
\newblock \emph{arXiv preprint arXiv:2103.13883}, 2021.

\bibitem[Fujimoto et~al.(2019)Fujimoto, Meger, and Precup]{fujimoto2019off}
Scott Fujimoto, David Meger, and Doina Precup.
\newblock Off-policy deep reinforcement learning without exploration.
\newblock In \emph{International Conference on Machine Learning}, pp.\
  2052--2062. PMLR, 2019.

\bibitem[Gottesman et~al.(2019)Gottesman, Johansson, Komorowski, Faisal,
  Sontag, Doshi-Velez, and Celi]{gottesman2019guidelines}
Omer Gottesman, Fredrik Johansson, Matthieu Komorowski, Aldo Faisal, David
  Sontag, Finale Doshi-Velez, and Leo~Anthony Celi.
\newblock Guidelines for reinforcement learning in healthcare.
\newblock \emph{Nature medicine}, 25\penalty0 (1):\penalty0 16--18, 2019.

\bibitem[Hanin \& Nica(2019)Hanin and Nica]{hanin2019finite}
Boris Hanin and Mihai Nica.
\newblock Finite depth and width corrections to the neural tangent kernel.
\newblock \emph{arXiv preprint arXiv:1909.05989}, 2019.

\bibitem[Hu et~al.(2021)Hu, Kallus, and Uehara]{DBLP:conf/colt/HuKU21}
Yichun Hu, Nathan Kallus, and Masatoshi Uehara.
\newblock Fast rates for the regret of offline reinforcement learning.
\newblock In \emph{{COLT}}, volume 134 of \emph{Proceedings of Machine Learning
  Research}, pp.\  2462. {PMLR}, 2021.

\bibitem[Jacot et~al.(2018)Jacot, Gabriel, and Hongler]{jacot2018neural}
Arthur Jacot, Franck Gabriel, and Cl{\'e}ment Hongler.
\newblock Neural tangent kernel: Convergence and generalization in neural
  networks.
\newblock \emph{arXiv preprint arXiv:1806.07572}, 2018.

\bibitem[Jin et~al.(2020)Jin, Yang, and
  Wang]{DBLP:journals/corr/abs-2012-15085}
Ying Jin, Zhuoran Yang, and Zhaoran Wang.
\newblock Is pessimism provably efficient for offline rl?
\newblock \emph{CoRR}, abs/2012.15085, 2020.
\newblock URL \url{https://arxiv.org/abs/2012.15085}.

\bibitem[Kingma \& Ba(2014)Kingma and Ba]{kingma2014adam}
Diederik~P Kingma and Jimmy Ba.
\newblock Adam: A method for stochastic optimization.
\newblock \emph{arXiv preprint arXiv:1412.6980}, 2014.

\bibitem[Kitagawa \& Tetenov(2018)Kitagawa and Tetenov]{Kitagawa18}
Toru Kitagawa and Aleksey Tetenov.
\newblock Who should be treated? empirical welfare maximization methods for
  treatment choice.
\newblock \emph{Econometrica}, 86\penalty0 (2):\penalty0 591--616, 2018.
\newblock \doi{https://doi.org/10.3982/ECTA13288}.
\newblock URL \url{https://onlinelibrary.wiley.com/doi/abs/10.3982/ECTA13288}.

\bibitem[Kleinberg et~al.(2008)Kleinberg, Slivkins, and
  Upfal]{kleinberg2008multi}
Robert Kleinberg, Aleksandrs Slivkins, and Eli Upfal.
\newblock Multi-armed bandits in metric spaces.
\newblock In \emph{Proceedings of the fortieth annual ACM symposium on Theory
  of computing}, pp.\  681--690, 2008.

\bibitem[Krause \& Ong(2011)Krause and Ong]{krause2011contextual}
Andreas Krause and Cheng~Soon Ong.
\newblock Contextual gaussian process bandit optimization.
\newblock In \emph{Nips}, pp.\  2447--2455, 2011.

\bibitem[Lange et~al.(2012)Lange, Gabel, and Riedmiller]{lange2012batch}
Sascha Lange, Thomas Gabel, and Martin Riedmiller.
\newblock Batch reinforcement learning.
\newblock In \emph{Reinforcement learning}, pp.\  45--73. Springer, 2012.

\bibitem[Langford \& Zhang(2007)Langford and Zhang]{langford2007epoch}
John Langford and Tong Zhang.
\newblock The epoch-greedy algorithm for contextual multi-armed bandits.
\newblock \emph{Advances in neural information processing systems}, 20\penalty0
  (1):\penalty0 96--1, 2007.

\bibitem[Lattimore \& Szepesvári(2020)Lattimore and
  Szepesvári]{lattimore_szepesvari_2020}
Tor Lattimore and Csaba Szepesvári.
\newblock \emph{Bandit Algorithms}.
\newblock Cambridge University Press, 2020.
\newblock \doi{10.1017/9781108571401}.

\bibitem[LeCun et~al.(1998)LeCun, Bottou, Bengio, and
  Haffner]{lecun1998gradient}
Yann LeCun, L{\'e}on Bottou, Yoshua Bengio, and Patrick Haffner.
\newblock Gradient-based learning applied to document recognition.
\newblock \emph{Proceedings of the IEEE}, 86\penalty0 (11):\penalty0
  2278--2324, 1998.

\bibitem[Levine et~al.(2020)Levine, Kumar, Tucker, and Fu]{levine2020offline}
Sergey Levine, Aviral Kumar, George Tucker, and Justin Fu.
\newblock Offline reinforcement learning: Tutorial, review, and perspectives on
  open problems.
\newblock \emph{arXiv preprint arXiv:2005.01643}, 2020.

\bibitem[Mnih et~al.(2015)Mnih, Kavukcuoglu, Silver, Rusu, Veness, Bellemare,
  Graves, Riedmiller, Fidjeland, Ostrovski, et~al.]{mnih2015human}
Volodymyr Mnih, Koray Kavukcuoglu, David Silver, Andrei~A Rusu, Joel Veness,
  Marc~G Bellemare, Alex Graves, Martin Riedmiller, Andreas~K Fidjeland, Georg
  Ostrovski, et~al.
\newblock Human-level control through deep reinforcement learning.
\newblock \emph{nature}, 518\penalty0 (7540):\penalty0 529--533, 2015.

\bibitem[Munos \& Szepesv{\'{a}}ri(2008)Munos and
  Szepesv{\'{a}}ri]{DBLP:journals/jmlr/MunosS08}
R{\'{e}}mi Munos and Csaba Szepesv{\'{a}}ri.
\newblock Finite-time bounds for fitted value iteration.
\newblock \emph{J. Mach. Learn. Res.}, 9:\penalty0 815--857, 2008.

\bibitem[Nguyen(2021)]{nguyen2021proof}
Quynh Nguyen.
\newblock On the proof of global convergence of gradient descent for deep relu
  networks with linear widths.
\newblock \emph{arXiv preprint arXiv:2101.09612}, 2021.

\bibitem[Nguyen-Tang et~al.(2021)Nguyen-Tang, Gupta, Tran-The, and
  Venkatesh]{nguyentang2021sample}
Thanh Nguyen-Tang, Sunil Gupta, Hung Tran-The, and Svetha Venkatesh.
\newblock Sample complexity of offline reinforcement learning with deep relu
  networks, 2021.

\bibitem[Nie et~al.(2021)Nie, Brunskill, and Wager]{nie2021learning}
Xinkun Nie, Emma Brunskill, and Stefan Wager.
\newblock Learning when-to-treat policies.
\newblock \emph{Journal of the American Statistical Association}, 116\penalty0
  (533):\penalty0 392--409, 2021.

\bibitem[Papini et~al.(2021)Papini, Tirinzoni, Restelli, Lazaric, and
  Pirotta]{papini2021leveraging}
Matteo Papini, Andrea Tirinzoni, Marcello Restelli, Alessandro Lazaric, and
  Matteo Pirotta.
\newblock Leveraging good representations in linear contextual bandits.
\newblock \emph{arXiv preprint arXiv:2104.03781}, 2021.

\bibitem[Rashidinejad et~al.(2021)Rashidinejad, Zhu, Ma, Jiao, and
  Russell]{rashidinejad2021bridging}
Paria Rashidinejad, Banghua Zhu, Cong Ma, Jiantao Jiao, and Stuart Russell.
\newblock Bridging offline reinforcement learning and imitation learning: A
  tale of pessimism.
\newblock \emph{arXiv preprint arXiv:2103.12021}, 2021.

\bibitem[Riquelme et~al.(2018)Riquelme, Tucker, and Snoek]{riquelme2018deep}
Carlos Riquelme, George Tucker, and Jasper Snoek.
\newblock Deep bayesian bandits showdown: An empirical comparison of bayesian
  deep networks for thompson sampling.
\newblock \emph{arXiv preprint arXiv:1802.09127}, 2018.

\bibitem[Srinivas et~al.(2009)Srinivas, Krause, Kakade, and
  Seeger]{srinivas2009gaussian}
Niranjan Srinivas, Andreas Krause, Sham~M Kakade, and Matthias Seeger.
\newblock Gaussian process optimization in the bandit setting: No regret and
  experimental design.
\newblock \emph{arXiv preprint arXiv:0912.3995}, 2009.

\bibitem[Strehl et~al.(2010)Strehl, Langford, Kakade, and
  Li]{strehl2010learning}
Alex Strehl, John Langford, Sham Kakade, and Lihong Li.
\newblock Learning from logged implicit exploration data.
\newblock \emph{arXiv preprint arXiv:1003.0120}, 2010.

\bibitem[Thomas et~al.(2017)Thomas, Theocharous, Ghavamzadeh, Durugkar, and
  Brunskill]{thomasAAAI17}
Philip~S. Thomas, Georgios Theocharous, Mohammad Ghavamzadeh, Ishan Durugkar,
  and Emma Brunskill.
\newblock Predictive off-policy policy evaluation for nonstationary decision
  problems, with applications to digital marketing.
\newblock In \emph{Proceedings of the Thirty-First AAAI Conference on
  Artificial Intelligence}, AAAI'17, pp.\  4740–4745. AAAI Press, 2017.

\bibitem[Tran-The et~al.(2021)Tran-The, Gupta, Nguyen-Tang, Rana, and
  Venkatesh]{tran2021combining}
Hung Tran-The, Sunil Gupta, Thanh Nguyen-Tang, Santu Rana, and Svetha
  Venkatesh.
\newblock Combining online learning and offline learning for contextual bandits
  with deficient support.
\newblock \emph{arXiv preprint arXiv:2107.11533}, 2021.

\bibitem[Uehara et~al.(2021)Uehara, Imaizumi, Jiang, Kallus, Sun, and
  Xie]{uehara2021finite}
Masatoshi Uehara, Masaaki Imaizumi, Nan Jiang, Nathan Kallus, Wen Sun, and
  Tengyang Xie.
\newblock Finite sample analysis of minimax offline reinforcement learning:
  Completeness, fast rates and first-order efficiency.
\newblock \emph{arXiv preprint arXiv:2102.02981}, 2021.

\bibitem[Valko et~al.(2013)Valko, Korda, Munos, Flaounas, and
  Cristianini]{valko2013finite}
Michal Valko, Nathaniel Korda, R{\'e}mi Munos, Ilias Flaounas, and Nelo
  Cristianini.
\newblock Finite-time analysis of kernelised contextual bandits.
\newblock \emph{arXiv preprint arXiv:1309.6869}, 2013.

\bibitem[Xiao et~al.(2021)Xiao, Wu, Mei, Dai, Lattimore, Li, Szepesvari, and
  Schuurmans]{xiao2021optimality}
Chenjun Xiao, Yifan Wu, Jincheng Mei, Bo~Dai, Tor Lattimore, Lihong Li, Csaba
  Szepesvari, and Dale Schuurmans.
\newblock On the optimality of batch policy optimization algorithms.
\newblock In \emph{International Conference on Machine Learning}, pp.\
  11362--11371. PMLR, 2021.

\bibitem[Xu et~al.(2020)Xu, Wen, Zhao, and Gu]{xu2020neural}
Pan Xu, Zheng Wen, Handong Zhao, and Quanquan Gu.
\newblock Neural contextual bandits with deep representation and shallow
  exploration.
\newblock \emph{arXiv preprint arXiv:2012.01780}, 2020.

\bibitem[Yang \& Wang(2020)Yang and Wang]{yang2020reinforcement}
Lin Yang and Mengdi Wang.
\newblock Reinforcement learning in feature space: Matrix bandit, kernels, and
  regret bound.
\newblock In \emph{International Conference on Machine Learning}, pp.\
  10746--10756. PMLR, 2020.

\bibitem[Yang et~al.(2020)Yang, Jin, Wang, Wang, and Jordan]{yang2020function}
Zhuoran Yang, Chi Jin, Zhaoran Wang, Mengdi Wang, and Michael~I Jordan.
\newblock On function approximation in reinforcement learning: Optimism in the
  face of large state spaces.
\newblock \emph{arXiv preprint arXiv:2011.04622}, 2020.

\bibitem[Yin \& Wang(2020)Yin and Wang]{DBLP:conf/aistats/YinW20}
Ming Yin and Yu{-}Xiang Wang.
\newblock Asymptotically efficient off-policy evaluation for tabular
  reinforcement learning.
\newblock In \emph{{AISTATS}}, volume 108 of \emph{Proceedings of Machine
  Learning Research}, pp.\  3948--3958. {PMLR}, 2020.

\bibitem[Yin \& Wang(2021)Yin and Wang]{yin2021characterizing}
Ming Yin and Yu-Xiang Wang.
\newblock Characterizing uniform convergence in offline policy evaluation via
  model-based approach: Offline learning, task-agnostic and reward-free, 2021.

\bibitem[Yin et~al.(2021)Yin, Bai, and Wang]{DBLP:conf/aistats/YinBW21}
Ming Yin, Yu~Bai, and Yu{-}Xiang Wang.
\newblock Near-optimal provable uniform convergence in offline policy
  evaluation for reinforcement learning.
\newblock In Arindam Banerjee and Kenji Fukumizu (eds.), \emph{The 24th
  International Conference on Artificial Intelligence and Statistics, {AISTATS}
  2021, April 13-15, 2021, Virtual Event}, volume 130 of \emph{Proceedings of
  Machine Learning Research}, pp.\  1567--1575. {PMLR}, 2021.
\newblock URL \url{http://proceedings.mlr.press/v130/yin21a.html}.

\bibitem[Zhan et~al.(2021)Zhan, Hadad, Hirshberg, and Athey]{zhan2021off}
Ruohan Zhan, Vitor Hadad, David~A Hirshberg, and Susan Athey.
\newblock Off-policy evaluation via adaptive weighting with data from
  contextual bandits.
\newblock In \emph{Proceedings of the 27th ACM SIGKDD Conference on Knowledge
  Discovery \& Data Mining}, pp.\  2125--2135, 2021.

\bibitem[Zhou et~al.(2020)Zhou, Li, and Gu]{zhou2020neural}
Dongruo Zhou, Lihong Li, and Quanquan Gu.
\newblock Neural contextual bandits with ucb-based exploration.
\newblock In \emph{International Conference on Machine Learning}, pp.\
  11492--11502. PMLR, 2020.

\end{thebibliography}
\newpage

\renewcommand{\thesection}{A}
\section{Proof of Theorem \ref{main_theorem}}
\label{section:main_proof}
In this section, we provide the proof of Theorem \ref{main_theorem}.

Let $\mathcal{D}_t = \{(\bm{x}_{\tau}, a_{\tau}, r_{\tau})\}_{1 \leq \tau \leq t}$. Note that $\hat{\pi}_t$ returned by Algorithm \ref{alg:neuralcb-sgd} is $\mathcal{D}_{t-1}$-measurable. Denote $\mathbb{E}_t[\cdot] = \mathbb{E}[\cdot | \mathcal{D}_{t-1}, \bm{x}_t]$. Let the step sizes $\{\eta_t\}$ defined as in Theorem \ref{main_theorem} and the confidence trade-off parameters $\{\beta_t\}$ defined as in Algorithm \ref{alg:neuralcb-sgd}, we present main lemmas below which will culminate into the proof of the main theorem. 

\begin{lem}
There exists absolute constants $C_1, C_2 > 0$ such that for any $\delta \in (0,1)$, if $m$ satisfies 
\begin{align*}
    m \geq \max \bigg \{ \Theta(n \lambda^{-1} L^{11} \log^6 m), \Theta(L^6 n^4 K^4 \lambda_0^{-4} \log(KL(5n+1)/\delta) )\\
    \Theta( L^{-1} \lambda^{1/2} (\log^{3/2}( nKL^2 (5n + 1)/ \delta) \lor \log^{-3/2} m ))
     \bigg \},
\end{align*}
then with probability at least $1 - \delta$, it holds uniformly over all $t \in [n]$ that 
\begin{align*}
    &\subopt(\hat{\pi}_t; \bm{x}_t) \leq  2\beta_{t-1} \mathbb{E}_{a^*_t \sim \pi^*(\cdot|\bm{x}_t)} \left[ \| \nabla f_{\bm{W}^{(t-1)}}(\bm{x}_{t,a^*_t} ) \cdot m^{-1/2} \|_{ \bm{\Lambda}^{-1}_{t-1} }  | \mathcal{D}_{t-1}, \bm{x}_t \right] \\ 
    &+ 2 C_1 t^{2/3} m^{-1/6} \log^{1/2} m L^{7/3} \lambda^{-1/2} + 2 \sqrt{2} C_2 \sqrt{nK} \lambda_0 ^{-1/2} m^{-11/6} \log^{1/2} m t^{1/6} L^{10/3} \lambda^{-1/6}.  
\end{align*}
\label{lemma:subopt_single_t_bound}
\end{lem}
Lemma \ref{lemma:subopt_single_t_bound} gives an upper bound on the sub-optimality of the returned step-dependent policy $\hat{\pi}_t$ on the observed context $\bm{x}_t$ for each $t \in [n]$. We remark that the upper bound depends on the rate at which the confidence width of the NTK feature vectors shrink along the direction of only the optimal actions, rather than any other actions. This is an advantage of pessimism where it does not require the offline data to be informative about any sub-optimal actions.  However, the upper bound depends on the unknown optimal policy $\pi^*$ while the offline data has been generated a priori by a different unknown behaviour policy. This distribution mismatch is handled in the next lemma. 
\begin{lem}
There exists an absolute constant $C_3 > 0 $ such that for any $\delta \in (0,1)$, if $m$ and $\lambda$ satisfy 
\begin{align*}
    \lambda \geq \max \{1, \Theta(L) \}, \hspace{0.5cm}  m \geq \max \bigg \{ \Theta( n \lambda^{-1} L^{11} \log^6 m), \Theta( L^6 n^4 K^4 \lambda_0^{-4} \log(nKL(5n+2)/\delta)) \\
    \Theta( L^{-1} \lambda^{1/2} (\log^{3/2}(nKL^2 (5n+2) / \delta) \lor \log^{-3/2} m ))
     \bigg \}, 
\end{align*}
then with probability at least $1 - \delta$, we have
\begin{align*}
    &\frac{1}{n} \sum_{t=1}^n \beta_{t-1} \mathbb{E}_{a^*_t \sim \pi^*(\cdot|\bm{x}_t)} \left[ \| \nabla f_{\bm{W}^{(t-1)}}(\bm{x}_{t, a^*_t} ) \cdot m^{-1/2} \|_{ \bm{\Lambda}^{-1}_{t-1} }  | \mathcal{D}_{t-1}, \bm{x}_t \right] \\
    &\leq \frac{\sqrt{2} \beta_n \kappa}{\sqrt{n}} \sqrt{ \tilde{d} \log(1 + nK/\lambda)  + 1 + 2 C_2 C_3^2   n^{3/2} m^{-1/6} (\log m)^{1/2} L^{23/6} \lambda^{-1/6}} \\
    &+ \beta_n \kappa ( C_3 / \sqrt{2}) L^{1/2} \lambda_0^{-1/2} \log^{1/2}((5n + 2)/\delta),
\end{align*}
where $C_2$ is from Lemma \ref{lemma:subopt_single_t_bound}. 
\label{lemma:bound_sum}
\end{lem}
We also remark that the upper bound in Lemma \ref{lemma:bound_sum} scales linearly with $\sqrt{\tilde{d}}$ instead of with $\sqrt{p}$ if a standard analysis were applied. This avoids a vacuous bound as $p$ is large with respect to $n$.

The upper bounds in Lemma \ref{lemma:subopt_single_t_bound} and Lemma \ref{lemma:bound_sum} are established for the observed contexts only. The next lemma generalizes these bounds to the entire context distribution, thanks to the online-like nature of Algorithm \ref{alg:neuralcb-sgd} and an online-to-batch argument. In particular, a key technical property of Algorithm \ref{alg:neuralcb-sgd} that makes this generalization possible without a uniform convergence is that  $\hat{\pi}_t$ is $\mathcal{D}_{t-1}$-measurable and independent of $(\bm{x}_t, a_t, r_t)$. 



\begin{lem}
For any $\delta \in (0,1)$, with probability at least $1-\delta$ over the randomness of $\mathcal{D}_n$, we have
     $\mathbb{E} \left[ \subopt(\hat{\pi}) \right] \leq \frac{1}{n} \sum_{t=1}^n \subopt(\hat{\pi}_t; \bm{x}_t) + \sqrt{ \frac{2}{n} \log(1/\delta) }. $
\label{lemma:online-to-batch}
\end{lem}
We are now ready to prove the main theorem. 
\begin{proof}[Proof of Theorem \ref{main_theorem}]
Combining Lemma \ref{lemma:subopt_single_t_bound}, Lemma \ref{lemma:bound_sum}, and Lemma \ref{lemma:online-to-batch} via the union bound, we have
\begin{align*}
    n \cdot \mathbb{E}[ \subopt(\hat{\pi}) ] &\leq  \kappa \sqrt{n}  \Gamma_1 \sqrt{\tilde{d} \log(1 + nK/\lambda)  + \Gamma_2} +  \kappa \sqrt{n} \Gamma_3 + \Gamma_4 + \Gamma_5 + \sqrt{2 n \log((10n+4)/\delta)}, \\ 
    &\leq  \kappa \sqrt{n}  \sqrt{\tilde{d} \log(1 + nK/\lambda)  + 2} +  \kappa \sqrt{n}  + 2 + \sqrt{2 n \log(10n+4)/\delta)}
\end{align*}
where $m$ is chosen to be sufficiently large as a polynomial of $(n,L, K, \lambda^{-1}, \lambda_0^{-1}, \log(1/\delta))$ such that 
\begin{align*}
    \Gamma_1 &:= 2 \sqrt{2} \sqrt{\lambda + C_3^2 n L} (n^{1/2} \lambda^{1/2} + (nK)^{1/2} \lambda_0^{-1/2}) \cdot m^{-1/2} \leq 1 \\
    \Gamma_2 &:= 1 + 2 C_2 C_3^2  n^{3/2} m^{-1/6} (\log m)^{1/2} L^{23/6} \lambda^{-1/6} \leq 2 \\ 
    \Gamma_3 &:= \Gamma_1 \sqrt{n} (C_3/\sqrt{2}) L^{1/2} \lambda_0^{-1/2} \log^{1/2}((10n + 4)/\delta) \leq 1 \\ 
    \Gamma_4 & := 2 C_1 n^{5/3} m^{-1/6} (\log m)^{1/2} L^{7/3} \lambda^{-1/2} \leq 1 \\ 
    \Gamma_5 &:= 2 \sqrt{2} C_2 (nK)^{1/2} \lambda_0^{-1/2} m^{-11/6} (\log m)^{1/2} n^{7/6} L^{10/3} \lambda^{-1/6} \leq 1.
\end{align*}
\end{proof}

\renewcommand{\thesection}{B}
\section{Proof of Lemmas in Section \ref{section:main_proof}}
\label{section:proof_of_sec6}
\subsection{Proof of Lemma \ref{lemma:subopt_single_t_bound}}

We start with the following lemmas whose proofs are deferred to Appendix B. 

\begin{lem}
Let $\bm{h} = [h(\bm{x}^{(1)}), \ldots, h(\bm{x}^{(nK)})]^T \in \mathbb{R}^{nK}$. There exists $\bm{W}^* \in \mathcal{W}$ such that for any $\delta \in (0,1)$, if $m \geq \Theta(L^6 n^4 K^4 \lambda_0^{-4} \log(nKL/\delta))$, with probability at least $1- \delta$ over the randomness of $\bm{W}^{(0)}$, it holds uniformly for all $ i \in [nK]$ that
\begin{align*}
    \| \bm{W}^* - \bm{W}^{(0)} \|_F \leq \sqrt{2} m^{-1/2} \| \bm{h}\|_{\bm{H}^{-1}},  \\
    \langle \nabla f_{\bm{W}^{(0)}}(\bm{x}^{(i)}), \bm{W}^* - \bm{W}^{(0)} \rangle = h(\bm{x}^{(i)}).
\end{align*}
\label{lemma:target_function_as_linear}
\end{lem}
\begin{rem}
Lemma \ref{lemma:target_function_as_linear} shows that for a sufficiently wide network, there is a linear model that uses the gradient of the neural network at initialization as a feature vector and interpolates the reward function in the training inputs. Moreover, the weights $\bm{W}^*$ of the linear model is in a neighborhood of the initialization $\bm{W}^{(0)}$. Note that we also have
\begin{align*}
    S := \| \bm{h} \|_{\bm{H}^{-1}} \leq \| \bm{h} \|_2 \sqrt{ \| \bm{H}^{-1} \|_2} \leq \sqrt{nK} \lambda_0 ^{-1/2},
\end{align*}
where the second inequality is by Assumption \ref{assumption:ntk_input_data} and Cauchy-Schwartz inequality with $h(\bm{x}) \in [0,1], \forall \bm{x}$. 
\end{rem}

\begin{lem} 
For any $\delta \in (0,1)$, if $m$ satisfies 
\begin{align*}
    m \geq \Theta(n \lambda^{-1} L^{11} \log^6 m ) \lor \Theta( L^{-1} \lambda^{1/2} \log^{3/2}(3n^2KL^2 / \delta)),
\end{align*}
and the step sizes satisfy 
\begin{align*}
    \eta_t = \frac{\iota}{\sqrt{t}} \text{ where } \iota^{-1} = \Omega(n^{2/3} m^{5/6} \lambda^{-1/6} L^{17/6} \log^{1/2} m) \lor \Omega(R m \lambda^{1/2} \log^{1/2}(n / \delta))
\end{align*}
then with probability at least $1 - \delta$ over the randomness of $\bm{W}^{(0)}$ and $\mathcal{D}$, it holds uniformly for all  $t \in [n], l \in [L]$  that 
\begin{align*}
     \|\bm{W}^{(t)}_l - \bm{W}^{(0)}_l \|_F \leq \sqrt{\frac{t}{m\lambda L }}, \text{ and } \| \bm{\Lambda}_t \|_2 \leq \lambda + C_3 ^2 tL,
\end{align*}
where $C_3 > 0$ is an absolute constant from Lemma \ref{lemma:bounded_gradient}. 
\label{lemma:W_t_away_from_init}
\end{lem}
\begin{rem}
Lemma \ref{lemma:W_t_away_from_init} controls the growth dynamics of the learned weights $\bm{W}_t$ around its initialization and bounds the spectral norm of the empirical covariance matrix $\bm{\Lambda}_t$ when the model is trained by SGD. 
\end{rem}

\begin{lem}[{\citet[Theorem~5]{allen2019convergence}, \citet[Lemma~B.5]{cao2019generalization}}]
There exist an absolute constant $C_2 > 0$ such that for any $\delta \in (0,1)$, if $\omega$ satisfies 
\begin{align*}
    \Theta( m^{-3/2} L^{-3/2} ( \log^{3/2}(nK/\delta)) \lor \log^{-3/2} m) \leq \omega \leq \Theta( L^{-9/2} \log^{-3} m), 
\end{align*}
with probability at least $1 - \delta$ over the randomness of $\bm{W}^{(0)}$, it holds uniformly for all $\bm{W} \in \mathcal{B}(\bm{W}^{(0)}; \omega)$ and $i \in [nK]$ that
\begin{align*}
    \| \nabla f_{\bm{W}}(\bm{x}^{(i)}) - \nabla f_{\bm{W}^{(0)}}(\bm{x}^{(i)}) \|_F \leq C_2 \sqrt{\log m} \omega^{1/3} L^3 \| \nabla f_{\bm{W}^{(0)}}(\bm{x}^{(i)})  \|_F.
\end{align*}
\label{lemma:grad_difference_away_from_init}
\end{lem}

\begin{rem}
Lemma \ref{lemma:grad_difference_away_from_init} shows that the gradient in a neighborhood of the initialization differs from the gradient at the initialization by an amount that can be explicitly controlled by the radius of the neighborhood and the norm of the gradient at initialization. 
\end{rem}

\begin{lem}[{\citet[Lemma~4.1]{cao2019generalization}}]
There exist an absolute constant $C_1 > 0$ such that for any $\delta \in (0,1)$ over the randomness of $\bm{W}^{(0)}$, if $\omega$ satisfies 
\begin{align*}
    \Theta( m^{-3/2} L^{-3/2} \log^{3/2}(nKL^2/\delta)) \leq \omega \leq \Theta( L^{-6} \log^{-3/2} m),
\end{align*}
with probability at least $1 - \delta$, it holds uniformly for all $\bm{W}, \bm{W}' \in \mathcal{B}(\bm{W}^{(0)}; \omega)$ and $i \in [nK]$ that 
\begin{align*}
    |f_{\bm{W}'}(\bm{x}^{(i)}) - f_{\bm{W}}(\bm{x}^{(i)}) - \langle \nabla f_{\bm{W}}(\bm{x}^{(i)}), \bm{W}' - \bm{W} \rangle | \leq C_1 \cdot \omega^{4/3} L^3 \sqrt{m \log m}.
\end{align*}
\label{lemma:nn_almost_linear_near_init_in_training_data}
\end{lem}
\begin{rem}
Lemma \ref{lemma:nn_almost_linear_near_init_in_training_data} shows that near initialization the neural network function is almost linear in terms of its weights in the training inputs. 
\end{rem}

\begin{proof}[Proof of Lemma \ref{lemma:subopt_single_t_bound}]
For all $t \in [n], \bm{u} \in \mathbb{R}^d$, we define 
\begin{align*}
    U_t(\bm{u}) &= f_{\bm{W}^{(t-1)}}(\bm{u})  + \beta_{t-1} \| \nabla f_{\bm{W}^{(t-1)}}(\bm{u} ) \cdot m^{-1/2} \|_{ \bm{\Lambda}^{-1}_{t-1} } \\
    L_t(\bm{u}) &= f_{\bm{W}^{(t-1)}}(\bm{u})  - \beta_{t-1} \| \nabla f_{\bm{W}^{(t-1)}}(\bm{u} ) \cdot m^{-1/2} \|_{ \bm{\Lambda}^{-1}_{t-1} } \\
    \bar{U}_t(\bm{u}) &=  \langle \nabla f_{\bm{W}^{(t-1)}}(\bm{u}), \bm{W}^{(t-1)} - \bm{W}^{(0)} \rangle  + \beta_{t-1} \| \nabla f_{\bm{W}^{(t-1)}}(\bm{u} ) \cdot m^{-1/2} \|_{ \bm{\Lambda}^{-1}_{t-1} } \\
    \bar{L}_t(\bm{u}) &=  \langle \nabla f_{\bm{W}^{(t-1)}}(\bm{u}), \bm{W}^{(t-1)} - \bm{W}^{(0)} \rangle  - \beta_{t-1} \| \nabla f_{\bm{W}^{(t-1)}}(\bm{u} ) \cdot m^{-1/2} \|_{ \bm{\Lambda}^{-1}_{t-1} } \\ 
    \mathcal{C}_t &= \{ \bm{W} \in \mathcal{W}: \| \bm{W} - \bm{W}^{(t-1)} \|_{\bm{\Lambda}_{t-1}} \leq \beta_{t-1} \}.
\end{align*}
Let $\mathcal{E}$ be the event in which Lemma \ref{lemma:target_function_as_linear}, Lemma \ref{lemma:grad_difference_away_from_init}, Lemma \ref{lemma:grad_difference_away_from_init} for all $\omega \in \left \{ \sqrt{\frac{i}{ m \lambda L}} : 1 \leq i \leq n \right\}$, and Lemma \ref{lemma:nn_almost_linear_near_init_in_training_data} for all  $\omega \in \left \{ \sqrt{\frac{i}{ m \lambda L}} : 1 \leq i \leq n \right\}$ hold simultaneously.


Under event $\mathcal{E}$, for all $t \in [n]$, we have 
\begin{align*}
    \| \bm{W}^* - \bm{W}^{(t)} \|_{\bm{\Lambda}_t} &\leq \| \bm{W}^* - \bm{W}^{(t)} \|_F \sqrt{ \|\bm{\Lambda}_t \|_2}\\
    &\leq (\| \bm{W}^* - \bm{W}^{(0)} \|_F + \| \bm{W}^{(t)} - \bm{W}^{(0)} \|_F) \sqrt{ \|\bm{\Lambda}_t \|_2} \\ 
    &\leq (\sqrt{2} m^{-1/2}S + t^{1/2} \lambda^{-1/2} m^{-1/2} ) \sqrt{\lambda + C_3^2 tL} = \beta_t,
\end{align*}
where the second inequality is by the triangle inequality, and the third inequality is by Lemma \ref{lemma:target_function_as_linear} and Lemma \ref{lemma:W_t_away_from_init}. Thus, $\bm{W}^* \in \mathcal{C}_t, \forall t \in [n]$.

Denoting $a^{*}_t \sim \pi^{*}(\cdot | \bm{x}_t)$ and $\hat{a}_t \sim \hat{\pi}_t(\cdot | \bm{x}_t)$ , under event $\mathcal{E}$, we have 
\begin{align*}
    &\text{SubOpt}(\hat{\pi}_t; \bm{x}_t) = \mathbb{E}_t[h(\bm{x}_{t, a^*_t})] - \mathbb{E}_t[h(\bm{x}_{t,\hat{a}_t})] \\ 
    &\overset{(a)}{=} \mathbb{E}_t \left[ \langle  \nabla f_{\bm{W}^{(0)}}(\bm{x}_{t,a^*_t}), \bm{W}^* - \bm{W}^{(0)} \rangle \right] - \mathbb{E}_t \left[ \langle \nabla f_{\bm{W}^{(0)}}(\bm{x}_{t,\hat{a}_t}), \bm{W}^* - \bm{W}^{(0)} \rangle \right] \\ 
    &\overset{(b)}{\leq} \mathbb{E}_t \left[ \langle  \nabla f_{\bm{W}^{(t-1)}}(\bm{x}_{t,a^*_t}), \bm{W}^* - \bm{W}^{(0)} \rangle \right] - \mathbb{E}_t \left[ \langle \nabla f_{\bm{W}^{(t-1)}}(\bm{x}_{t,\hat{a}_t}), \bm{W}^* - \bm{W}^{(0)} \rangle \right] \\
    &+  \| \bm{W}^* - \bm{W}^{(0)} \|_F \cdot \mathbb{E}_t \bigg[ \|\nabla f_{\bm{W}^{(t-1)}}(\bm{x}_{t, a^*_t}) - \nabla f_{\bm{W}^{(0)}}(\bm{x}_{t, a^*_t}) \|_F \\
    & + \|\nabla f_{\bm{W}^{(t-1)}}(\bm{x}_{t, \hat{a}_t}) - \nabla f_{\bm{W}^{(0)}}(\bm{x}_{t, \hat{a}_t}) \|_F  \bigg] \\ 
    &\overset{(c)}{\leq} \mathbb{E}_t \left[ \langle  \nabla f_{\bm{W}^{(t-1)}}(\bm{x}_{t, a^*_t}), \bm{W}^* - \bm{W}^{(0)} \rangle  \right] - \mathbb{E}_t \left[ \langle \nabla f_{\bm{W}^{(t-1)}}(\bm{x}_{t, \hat{a}_t}), \bm{W}^* - \bm{W}^{(0)} \rangle \right] \\ 
    &+ 2 \sqrt{2} C_2 S m^{-11/6} \log^{1/2} m t^{1/6} L^{10/3} \lambda^{-1/6} \\ 
    &\overset{(d)}{\leq} \mathbb{E}_t \left[ \bar{U}_t(\bm{x}_{t, a^*_t}) \right] - \mathbb{E}_t  \left[ \bar{L}_t(\bm{x}_{t, \hat{a}_t}) \right] + 2 \sqrt{2} C_2 S m^{-11/6} \log^{1/2} m t^{1/6} L^{10/3} \lambda^{-1/6} \\ 
    &= \mathbb{E}_t  \left[ U_t(\bm{x}_{t, a^*_t})\right] - \mathbb{E}_t \left[ L_t(\bm{x}_{t, \hat{a}_t}) \right] \\
    &+ \mathbb{E}_t \left[ \langle \nabla f_{\bm{W}^{(t-1)}}(\bm{x}_{t,a^*_t}), \bm{W}^{(t-1)} - \bm{W}^{(0)} \rangle - f_{\bm{W}^{(t-1)}}(\bm{x}_{t, a^*_t}) + f_{\bm{W}^{(0)}}(\bm{x}_{t,a^*_t}) \right]  \\
    &+ \mathbb{E}_t \left[ f_{\bm{W}^{(t-1)}}(\bm{x}_{t,\hat{a}_t}) - f_{\bm{W}^{(0)}}(\bm{x}_{t,\hat{a}_t}) - \langle \nabla f_{\bm{W}^{(t-1)}}(\bm{x}_{t, \hat{a}_t}), \bm{W}^{(t-1)} - \bm{W}^{(0)} \rangle  \right] \\
    &\underbrace{- \mathbb{E}_t \left[f_{\bm{W}^{(0)}}(\bm{x}_{t, a^*_t}) \right] + \mathbb{E}_t \left[f_{\bm{W}^{(0)}}(\bm{x}_{t, \hat{a}_t}) \right]}_{=0 \text{ by symmetry at initialization}} + 2 \sqrt{2} C_2 S m^{-11/6} \log^{1/2} m t^{1/6} L^{10/3} \lambda^{-1/6} \\ 
    &\overset{(e)}{\leq} \mathbb{E}_t  \left[ U_t(\bm{x}_{t, a^*_t})\right] - \mathbb{E}_t \left[ L_t(\bm{x}_{t, a^*_t}) \right] + \underbrace{\left(\mathbb{E}_t \left[ L_t(\bm{x}_{t, a^*_t}) \right] - \mathbb{E}_t \left[ L_t(\bm{x}_{t, \hat{a}_t}) \right] \right)}_{\leq 0 \text{ by pessimism}} \\ 
    &+ 2 C_1 t^{2/3} m^{-1/6} \log^{1/2} m L^{7/3} \lambda^{-1/2} + 2 \sqrt{2} C_2 S m^{-11/6} \log^{1/2} m t^{1/6} L^{10/3} \lambda^{-1/6} \\ 
    &\overset{(f)}{\leq} \mathbb{E}_t  \left[ U_t(\bm{x}_{t, a^*_t})\right] - \mathbb{E}_t \left[ L_t(\bm{x}_{t, a^*_t}) \right]  \\ 
    &+ 2 C_1 t^{2/3} m^{-1/6} \log^{1/2} m L^{7/3} \lambda^{-1/2} + 2 \sqrt{2} C_2 S m^{-11/6} \log^{1/2} m t^{1/6} L^{10/3} \lambda^{-1/6} \\
    &= 2\beta_{t-1} \mathbb{E}_t \left[ \| \nabla f_{\bm{W}^{(t-1)}}(\bm{x}_{t, a^*_t} ) \cdot m^{-1/2} \|_{ \bm{\Lambda}^{-1}_{t-1} }  \right] \\ 
    &+ 2 C_1 t^{2/3} m^{-1/6} \log^{1/2} m L^{7/3} \lambda^{-1/2} + 2 \sqrt{2} C_2 S m^{-11/6} \log^{1/2} m t^{1/6} L^{10/3} \lambda^{-1/6} 
\end{align*}
where $(a)$ is by Lemma \ref{lemma:target_function_as_linear}, $(b)$ is by the triangle inequality, $(c)$ is by Lemma \ref{lemma:target_function_as_linear}, Lemma \ref{lemma:W_t_away_from_init}, and Lemma \ref{lemma:grad_difference_away_from_init}, $(d)$ is by  $\bm{W}^* \in \mathcal{C}_t$, and by that $\max_{ \bm{u}: \| \bm{u} - \bm{b} \|_{\bm{A}} \leq \gamma } \langle \bm{a}, \bm{u} - \bm{b}_0 \rangle = \langle \bm{a}, \bm{b} - \bm{b}_0 \rangle + \gamma \| \bm{a} \|_{\bm{A}^{-1}}$, and $\min_{ \bm{u}: \| \bm{u} - \bm{b} \|_{\bm{A}} \leq \gamma } \langle \bm{a}, \bm{u} - \bm{b}_0 \rangle = \langle \bm{a}, \bm{b} - \bm{b}_0 \rangle - \gamma \| \bm{a} \|_{\bm{A}^{-1}}$, $(e)$ is by Lemma \ref{lemma:nn_almost_linear_near_init_in_training_data} and by that $f_{\bm{W}^{(0)}}(\bm{x}^{(i)}) = 0, \forall i \in [nK]$, and $(f)$ is by that $\hat{a}_t$ is sampled from the policy $\hat{\pi}_t$ which is greedy with respect to $L_t$.

By the union bound and the choice of $m$, we conclude our proof.
\end{proof}

\subsection{Proof of Lemma \ref{lemma:bound_sum}}
We first present the following lemma. 
\begin{lem}
For any $\delta \in (0,1)$, if $m$ satisfies 
\begin{align*}
    m \geq \max \bigg \{ \Theta( n \lambda^{-1} L^{11} \log^6 m), \Theta( L^{-1} \lambda^{1/2} (\log^{3/2}(nKL^2 (n+2) / \delta) \lor \log^{-3/2} m) ), \\
    \Theta( L^6 (nK)^4 \log(L(n+2)/\delta)) \bigg \},
\end{align*}
and $\lambda \geq \max\{ C_3^2 L, 1 \}$, then with probability at least $1 - \delta$, it holds simultaneously that 
\begin{align*}
    \sum_{i=1}^t  \| \nabla f_{\bm{W}^{(i-1)}}(\bm{x}_{i,a_i}  ) \cdot m^{-1/2} \|_{ \bm{\Lambda}^{-1}_{i-1} }^2 &\leq 2 \log \frac{\det (\bm{\Lambda}_t)}{ \det (\lambda \bm{I})}, \forall t \in [n], \\
    \bigg | \log \frac{\det (\bm{\Lambda}_t)}{ \det (\lambda \bm{I})} - \log \frac{ \det (\bar{\bm{\Lambda}}_t)}{ \det (\lambda \bm{I})}\bigg| &\leq 2 C_2 C_3^2 t^{3/2} m^{-1/6} (\log m)^{1/2} L^{23/6} \lambda^{-1/6}, \forall t \in [n], \\ 
    \log \frac{\det(\bar{\bm{\Lambda}}_n)}{\det(\lambda \bm{I})} &\leq \tilde{d} \log(1 + nK/\lambda) + 1,
\end{align*}
where $\bar{\bm{\Lambda}}_t := \lambda \bm{I} + \sum_{i=1}^t \vect (\nabla f_{\bm{W}^{(0)}}(\bm{x}_{i, a_i}) ) \cdot \vect(\nabla f_{\bm{W}^{(0)}}(\bm{x}_{i, a_i}) )^T / m$, and $C_2, C_3 > 0$ are absolute constants from Lemma \ref{lemma:grad_difference_away_from_init} and Lemma \ref{lemma:W_t_away_from_init}, respectively. 
\label{lemma:bound_confidence_direction}
\end{lem}

We are now ready to prove Lemma \ref{lemma:bound_sum}. 
\begin{proof}[Proof of Lemma \ref{lemma:bound_sum}]
First note that $\| \bm{\Lambda}_{t-1} \|_2 \geq \lambda, \forall t$. Let $\mathcal{E}$ be the event in which Lemma \ref{lemma:W_t_away_from_init}, Lemma \ref{lemma:bounded_gradient} for all $\omega \in \left \{ \sqrt{\frac{i}{ m \lambda L}} : 1 \leq i \leq n \right\}$, and Lemma \ref{lemma:bound_confidence_direction} simultaneously hold. Thus, under event $\mathcal{E}$, we have
\begin{align*}
    \| \nabla f_{\bm{W}^{(t-1)}}(\bm{x}_{t, a_t} ) \cdot m^{-1/2} \|_{ \bm{\Lambda}^{-1}_{t-1} } &\leq \| \nabla f_{\bm{W}^{(t-1)}}(\bm{x}_{t,a_t} ) \cdot m^{-1/2} \|_F  \sqrt{ \| \bm{\Lambda}^{-1}_{t-1} \|_2 } \leq C_3 L^{1/2} \lambda^{-1/2},
\end{align*}
where the second inequality is by Lemma \ref{lemma:W_t_away_from_init} and Lemma \ref{lemma:bounded_gradient}.

Thus, by Assumption \ref{assumption:distributional_shift}, Hoeffding's inequality, and the union bound, with probability at least $1 - \delta$, it holds simultaneously for all $t \in [n]$ that  
\begin{align*}
    &2\beta_{t-1} \mathbb{E}_{a^*_t \sim \pi^{*}(\cdot| \bm{x}_t)}\left[ \| \nabla f_{\bm{W}^{(t-1)}}(\bm{x}_{t, a^*_t} ) \cdot m^{-1/2} \|_{ \bm{\Lambda}^{-1}_{t-1} }  | \mathcal{D}_{t-1}, \bm{x}_t \right] \\
    &\leq 2\beta_{t-1} \kappa \mathbb{E}_{a \sim \mu(\cdot|\mathcal{D}_{t-1}, \bm{x}_t)} \left[ \| \nabla f_{\bm{W}^{(t-1)}}(\bm{x}_{t,a} ) \cdot m^{-1/2} \|_{ \bm{\Lambda}^{-1}_{t-1} } | \mathcal{D}_{t-1}, \bm{x}_t \right] \\
    &\leq 2\beta_{t-1} \kappa  \| \nabla f_{\bm{W}^{(t-1)}}(\bm{x}_{t, a_t} ) \cdot m^{-1/2} \|_{ \bm{\Lambda}^{-1}_{t-1} }  + \beta_{t-1} \kappa  \sqrt{2} C_3 L^{1/2} \lambda^{-1/2} \log^{1/2}((5n+2)/\delta).
\end{align*}
Hence, for the choice of $m$ in Lemma \ref{lemma:bound_sum}, with probability at least $1 - \delta$, we have
\begin{align*}
     &\frac{1}{n} \sum_{t=1}^n \beta_{t-1} \mathbb{E}_{a^*_t \sim \pi^*(\cdot|\bm{x}_t)} \left[ \| \nabla f_{\bm{W}^{(t-1)}}(\bm{x}_{t,a^*_t} ) \cdot m^{-1/2} \|_{ \bm{\Lambda}^{-1}_{t-1} }  | \mathcal{D}_{t-1}, \bm{x}_t \right] \\
     &\leq \frac{\beta_n \kappa}{n}   \sum_{t=1}^n  \| \nabla f_{\bm{W}^{(t-1)}}(\bm{x}_{t,a_t}  ) \cdot m^{-1/2} \|_{ \bm{\Lambda}^{-1}_{t-1} }   +   \frac{C_3}{\sqrt{2}} \beta_n \kappa  L^{1/2} \lambda_0^{-1/2} \log^{1/2}((5n + 2)/\delta) \\ 
     &\leq \frac{\beta_n \kappa}{n}  \sqrt{n} \sqrt{ \sum_{t=1}^n  \| \nabla f_{\bm{W}^{(t-1)}}(\bm{x}_{t,a_t}  ) \cdot m^{-1/2} \|_{ \bm{\Lambda}^{-1}_{t-1} }^2 } +   \frac{C_3}{\sqrt{2}} \beta_n \kappa  L^{1/2} \lambda_0^{-1/2} \log^{1/2}((5n+2)/\delta) \\
     &\leq \frac{\sqrt{2} \beta_n \kappa}{ \sqrt{n} } \sqrt{\log \frac{\det (\bm{\Lambda}_n)}{ \det(\lambda \bm{I})}} +  \frac{C_3}{\sqrt{2}} \beta_n \kappa  L^{1/2} \lambda_0^{-1/2} \log^{1/2}((5n + 2)/\delta) \\
     &\leq \frac{\sqrt{2} \beta_n \kappa}{\sqrt{n}} \sqrt{ \log \frac{\det (\bar{\bm{\Lambda}}_n)}{ \det(\lambda \bm{I})}  + 2 C_2 C_3^2   n^{3/2} m^{-1/6} (\log m)^{1/2} L^{23/6} \lambda^{-1/6} } +  \frac{C_3}{\sqrt{2}} \beta_n \kappa  L^{1/2} \lambda_0^{-1/2} \log^{1/2}((5n+2)/\delta) \\ 
     &\leq \frac{\sqrt{2} \beta_n \kappa}{\sqrt{n}} \sqrt{ \tilde{d} \log(1 + nK/\lambda)  + 1 + 2 C_2 C_3^2   n^{3/2} m^{-1/6} (\log m)^{1/2} L^{23/6} \lambda^{-1/6}} \\
     &+ \frac{C_3}{\sqrt{2}} \beta_n \kappa  L^{1/2} \lambda_0^{-1/2} \log^{1/2}((5n+2)/\delta), 
\end{align*}
where the first inequality is by $\beta_t \leq \beta_n, \forall t \in [n]$, the second inequality is by Cauchy-Schwartz inequality, the second inequality and the third inequality are by Lemma \ref{lemma:bound_confidence_direction}.

\end{proof}

\subsection{Proof of Lemma \ref{lemma:online-to-batch}} 
\begin{proof}[Proof of Lemma \ref{lemma:online-to-batch}]
 
We follow the same online-to-batch conversion argument in \citep{DBLP:journals/tit/Cesa-BianchiCG04}. For each $t \in [n]$, define 
\begin{align*}
    Z_t = \text{SubOpt}(\hat{\pi}_t) - \text{SubOpt}(\hat{\pi}_t; \bm{x}_t).
\end{align*}
Since $\hat{\pi}_t$ is $\mathcal{D}_{t-1}$-measurable and is independent of $\bm{x}_t$, and $\bm{x}_t$ are independent of $\mathcal{D}_{t-1}$ (by Assumption \ref{assumption:distributional_shift}), we have $\mathbb{E} \left[  Z_t | \mathcal{D}_{t-1} \right] = 0, \forall t \in [n]$. Note that $-1 \leq Z_t \leq 1$. Thus, by the Hoeffding-Azuma inequality, with probability at least $1-\delta$, we have 
\begin{align*}
    \mathbb{E}[\subopt(\hat{\pi})] &= \frac{1}{n} \sum_{t=1}^n \subopt(\hat{\pi}_t) = \frac{1}{n} \sum_{t=1}^n \subopt(\hat{\pi}_t; \bm{x}_t) + \frac{1}{n} \sum_{t=1}^n Z_t \\
    &\leq \frac{1}{n} \sum_{t=1}^n \subopt(\hat{\pi}_t; \bm{x}_t) + \sqrt{ \frac{2}{n} \log(1/\delta) }.
\end{align*}
\end{proof}

\renewcommand{\thesection}{C}
\section{Proof of Lemmas in Section \ref{section:proof_of_sec6}}

\subsection{Proof of Lemma \ref{lemma:target_function_as_linear}}
We first restate the following lemma. 
\begin{lem}[\cite{arora2019exact}]
There exists an absolute constant $c_1 > 0$ such that for any $\epsilon > 0, \delta \in (0,1)$, if $m \geq c_1 L^6 \epsilon^{-4} \log(L/\delta) $, for any $i,j \in [nK]$, with probability at least $1-\delta$ over the randomness of $\bm{W}^{(0)}$, we have 
\begin{align*}
    |\langle \nabla f_{\bm{W}^{(0)}}(\bm{x}^{(i)}), \nabla f_{\bm{W}^{(0)}} (\bm{x}^{(j)}) \rangle /m - \bm{H}_{i,j} | \leq \epsilon.
\end{align*}
\label{lemma:empirical_ntk}
\end{lem}
Lemma \ref{lemma:empirical_ntk} gives an estimation error between the kernel constructed by the gradient at initialization as a feature map and the NTK kernel. Unlike \cite{jacot2018neural}, Lemma \ref{lemma:empirical_ntk} quantifies an exact non-asymptotic bound for $m$. 

\begin{proof}[Proof of Lemma \ref{lemma:target_function_as_linear}]
Let $\bm{G} = m^{-1/2} \cdot [\vect(\nabla f_{\bm{W}^{(0)}}(\bm{x}^{(1)})), \ldots, \vect(\nabla f_{\bm{W}^{(0)}}(\bm{x}^{(nK)}))] \in \mathbb{R}^{p \times nK}$. For any $\epsilon > 0, \delta \in (0,1)$,  it follows from Lemma \ref{lemma:empirical_ntk} and union bound, if $m \geq \Theta( L^6 \epsilon^{-4} \log(nKL/\delta))$, with probability at least $1 - \delta$, we have
\begin{align*}
    \|\bm{G}^T \bm{G} - \bm{H} \|_F  &\leq nK   \|\bm{G}^T \bm{G} - \bm{H} \|_{\infty} = nK \max_{i,j}|m^{-1} \langle \nabla f_{\bm{W}^{(0)}}(\bm{x}^{(i)}), \nabla f_{\bm{W}^{(0)}} (\bm{x}^{(j)}) \rangle - \bm{H}_{i,j} | \\
    &\leq nK \epsilon.
\end{align*}
Under the event that the inequality above holds, by setting $\epsilon = \frac{\lambda_0}{ 2 nK}$, we have 
\begin{align}
    \bm{H} - \bm{G}^T \bm{G} \preceq \| \bm{H} - \bm{G}^T \bm{G} \|_2 I  \preceq \| \bm{H} - \bm{G}^T \bm{G} \|_F I \preceq \frac{\lambda_0}{2} I \preceq \frac{1}{2} \bm{H}. 
    \label{eq:ntk_to_empirical_kernel}
\end{align}
Let $\bm{G} = \bm{P} \bm{\Lambda} \bm{Q}^T$ be the singular value decomposition of $\bm{G}$ where $\bm{P} \in \mathbb{R}^{p \times nK}, \bm{Q} \in \mathbb{R}^{nK \times nK}$ have orthogonal columns, and $\bm{\Lambda} \in \mathbb{R}^{nK \times nK}$ is a diagonal matrix. Since $\bm{G}^T \bm{G} \succeq \frac{1}{2} \bm{H} \succeq \frac{\lambda_0}{2} I $ is positive definite, $\bm{\Lambda}$ is invertible. Let $\bm{W}^* \in \mathcal{W}$ such that $\vect(\bm{W}^*) = \vect(\bm{W}^{(0)}) +  m^{-1/2} \cdot \bm{P} \bm{\Lambda}^{-1} \bm{Q}^T \bm{h}$, we have 
\begin{align*}
    m^{1/2} \cdot \bm{G}^T (\vect(\bm{W}^*) - \vect(\bm{W}^{(0)})) &= \bm{Q} \bm{\Lambda} \bm{P}^T \bm{P} \bm{\Lambda}^{-1} \bm{Q}^T \bm{h} = \bm{h}. 
\end{align*}
Moreover, we have 
\begin{align*}
m  \| \bm{W}^* - \bm{W}^{(0)} \|_F^2 &= m  \| \vect(\bm{W}^*) - \vect(\bm{W}^{(0)}) \|_2^2 = \bm{h}^T \bm{Q} \bm{\Lambda}^{-1} \bm{P}^T \bm{P} \bm{\Lambda}^{-1} \bm{Q}^T \bm{h} \\
&=\bm{h}^T \bm{Q} \bm{\Lambda}^{-2} \bm{Q}^T \bm{h} = \bm{h}^T (\bm{G}^T \bm{G})^{-1} \bm{h} \leq 2 \bm{h}^T \bm{H}^{-1} \bm{h},
\end{align*}
where the inequality is by Equation (\ref{eq:ntk_to_empirical_kernel}).
\end{proof}

\subsection{Proof of Lemma \ref{lemma:W_t_away_from_init}}
We present the following lemma that will be used in this proof. 

\begin{lem}[{\citet[Lemma~B.3]{cao2019generalization}}]
There exist an absolute constant $C_3 > 0$ such that for any $\delta \in (0,1)$ over the randomness of $\bm{W}^{(0)}$, if $\omega$ satisfies 
\begin{align*}
    \Theta(m^{-3/2} L^{-3/2} \log^{3/2}(nKL^2/\delta)) \leq \omega \leq \Theta( L^{-6} \log^{-3} m),
\end{align*}
with probability at least $1 - \delta$, it holds uniformly for all $\bm{W} \in \mathcal{B}(\bm{W}^{(0)}; \omega)$, $i \in [nK], l \in [L]$ that 
\begin{align*}
    \| \nabla_l f_{\bm{W}}(\bm{x}^{(i)}) \|_F \leq C_3 \cdot \sqrt{m}.
\end{align*}
\label{lemma:bounded_gradient}
\end{lem}

\begin{proof}[Proof of Lemma \ref{lemma:W_t_away_from_init}]
Let $\delta \in (0,1)$. Let $L_t(\bm{W}) = \frac{1}{2}(f_{\bm{W}}(\bm{x}_{t, a_t}) - r_t)^2  + \frac{m \lambda}{2} \| \bm{W} - \bm{W}^{(0)} \|_F^2$ be the regularized squared loss function on the data point $(\bm{x}_{t, a_t}, r_t)$. Recall that $\bm{W}^{(t)} = \bm{W}^{(t-1)} - \eta_t \nabla L_t(\bm{W}^{(t-1)})$. By Hoelfding's inequality and that $r_t$ is $R$-subgaussian, for any $t \in [n]$, with probability at least $1 - \delta$, we have 
\begin{align}
    |r_t| &\leq |\mathbb{E}_t[r_t]| + R \sqrt{2 \log(2/\delta)} = |\mathbb{E} \left[ \mathbb{E} [r_t | \mathcal{D}_{t-1}, \bm{x}_t, a_t ] \right]| + R \sqrt{2 \log(2/\delta)} \nonumber\\
    &= \mathbb{E} \left[ h(\bm{x}_{t, a_t}) |  \mathcal{D}_{t-1}, \bm{x}_t, a_t \right ] + R \sqrt{2 \log(2/\delta)} \leq 1 + R \sqrt{2 \log(2/\delta)}.
    \label{eq:bound_r_t}
\end{align}

By union bound and (\ref{eq:bound_r_t}), for any sequence $\{\omega_t\}_{t \in [n]}$ such that $\Theta( m^{-3/2} L^{-3/2} \log^{3/2}(3n^2KL^2/\delta)) \leq \omega_t \leq \Theta( L^{-6} \log^{-3} m ) \land \Theta( L^{-6} \log^{-3/2} m), \forall t \in [n]$, with probability at least $1 - \delta$, it holds uniformly for all $\bm{W} \in \mathcal{B}(\bm{W}^{(0)}; \omega_t), l \in [L], t \in [n]$ that
\begin{align}
    \|\nabla_l L_t(\bm{W})\|_F &= \| \nabla_l f_{\bm{W}}(\bm{x}_{t, a_t}) (f_{\bm{W}}(\bm{x}_{t, a_t}) - r_t) + m\lambda (\bm{W}_l - \bm{W}^{(0)}_l) \|_F \nonumber\\
    &= \| \nabla_l f_{\bm{W}}(\bm{x}_{t, a_t}) (f_{\bm{W}}(\bm{x}_{t, a_t}) - f_{\bm{W}^{(0)}}(\bm{x}_{t, a_t}) - \langle \nabla f_{\bm{W}^{(0)}}(\bm{x}_{t, a_t}), \bm{W} - \bm{W}^{(0)} \rangle) \nonumber\\ 
    &+ \nabla_l f_{\bm{W}}(\bm{x}_{t, a_t}) f_{\bm{W}^{(0)}}(\bm{x}_{t, a_t}) + \nabla_l f_{\bm{W}}(\bm{x}_{t, a_t}) \langle \nabla f_{\bm{W}^{(0)}}(\bm{x}_{t, a_t}), \bm{W} - \bm{W}^{(0)} \rangle \nonumber\\ 
    &- r_t \nabla_l f_{\bm{W}}(\bm{x}_{t, a_t}) + m\lambda (\bm{W}_l - \bm{W}^{(0)}_l) \|_F \nonumber\\ 
    &\leq \| \nabla_l f_{\bm{W}}(\bm{x}_{t, a_t}) (f_{\bm{W}}(\bm{x}_{t, a_t}) - f_{\bm{W}^{(0)}}(\bm{x}_{t, a_t}) - \langle \nabla f_{\bm{W}^{(0)}}(\bm{x}_{t, a_t}), \bm{W} - \bm{W}^{(0)} \rangle) \|_F \nonumber \\ 
    &+ \underbrace{\|\nabla_l f_{\bm{W}}(\bm{x}_{t, a_t}) f_{\bm{W}^{(0)}}(\bm{x}_{t, a_t})\|_F}_{=0} + \|\nabla_l f_{\bm{W}}(\bm{x}_{t, a_t}) \langle \nabla f_{\bm{W}^{(0)}}(\bm{x}_{t, a_t}), \bm{W} - \bm{W}^{(0)} \rangle \|_F \nonumber\\ 
    &+ |r_t|  \| \nabla_l f_{\bm{W}}(\bm{x}_{t, a_t}) \|_F + m\lambda \| \bm{W}_l - \bm{W}^{(0)}_l \|_F \nonumber \\
    &\leq C_1 C_3 \omega_t^{4/3} L^3 m \log^{1/2} m + C_3 m^{1/2} (1 + R \sqrt{2 \log(6n/\delta)}) + C_3^2 L m \omega + m \lambda \omega_t,
    \label{eq:loss_grad_bound}
\end{align}
where the first inequality is triangle's inequality, and the second inequality is by Lemma \ref{lemma:nn_almost_linear_near_init_in_training_data}, Lemma \ref{lemma:bounded_gradient} and (\ref{eq:bound_r_t}). 


We now prove by induction that under the same event that (\ref{eq:loss_grad_bound}) with $\omega_t = \sqrt{\frac{t}{m \lambda L }}$ holds, 
$\bm{W}^{(t)} \in \mathcal{B}(\bm{W}^{(0)}; \sqrt{\frac{t}{m \lambda L }}), \forall t \in [n]$. It trivially holds for $t = 0$. Assume $\bm{W}^{(i)} \in \mathcal{B}(\bm{W}^{(0)}; \sqrt{\frac{i}{m \lambda L }}), \forall i \in [t-1]$, we will prove that $\bm{W}^{(t)} \in \mathcal{B}(\bm{W}^{(0)}; \sqrt{\frac{t}{m \lambda L }})$. Indeed, it is easy to verify that there exist absolute constants $\{C_i\}_{i=1}^2 > 0$ such that if $m$ satisfies the inequalities in Lemma \ref{lemma:W_t_away_from_init}, then $\Theta( m^{-3/2} L^{-3/2} \log^{3/2}(3n^2KL^2/\delta)) \leq \sqrt{\frac{i}{m\lambda L}} \leq \Theta( L^{-6} \log^{-3} m) \land \Theta( L^{-6} \log^{-3/2} m), \forall i \in [n]$. Thus, under the same event that (\ref{eq:loss_grad_bound}) holds, we have
\begin{align*}
    \| \bm{W}^{(t)}_l -  \bm{W}^{(0)}_l\|_F &\leq \sum_{i=1}^t \| \bm{W}^{(i)}_l -  \bm{W}^{(i-1)}_l\|_F =  \sum_{i=1}^t \eta_i \|  \nabla_l  L_i(\bm{W}^{(i-1)}) \|_F \\ 
    &\leq \iota t C_1 C_3 m^{1/3} \lambda^{-2/3} L^{7/3} n^{1/6} \log^{1/2} m + \iota t m^{-1/2} \lambda^{-1/2} L^{-1/2} (C_3^2 L + \lambda) \\
    &+ 2 C_3 \iota  \sqrt{t} m^{1/2} (1 +  R \sqrt{2} \log^{1/2}(6n/\delta)) \\ 
    &\leq \frac{1}{2}\sqrt{\frac{t}{m \lambda L}} + \frac{1}{2}\sqrt{\frac{t}{m \lambda L}} = \sqrt{\frac{t}{m \lambda L}},
\end{align*}
where the first inequality is by triangle's inequality, the first equation is by the SGD update for each $\bm{W}^{(i)}$, the second inequality is by (\ref{eq:loss_grad_bound}) and the last inequality is due to $\iota$ be chosen as 
\begin{align*}
    \begin{cases}
    \iota^{-1} \geq 4 C_3 m \lambda^{1/2} L^{1/2} \left( 1 +R\sqrt{2} \log^{1/2}(6n/\delta) \right) \\ 
    \iota^{-1} \geq 2 n^{1/2} m^{1/2} \lambda^{1/2} L^{1/2} \left( C_1 C_3 m^{1/3} \lambda^{-2/3} L^{7/3} n^{1/6} \log^{1/2}m + m^{-1/2} \lambda^{-1/2} L^{-1/2}(C_3^2 L + \lambda)\right),
    \end{cases}
\end{align*}
which is satisfied for $\iota^{-1} = \Omega(n^{2/3} m^{5/6} \lambda^{-1/6} L^{17/6} \log^{1/2} m) \lor \Omega(R m \lambda^{1/2} \log^{1/2}( n / \delta))$.

For the second part of the lemma, we have 
\begin{align*}
    \| \bm{\Lambda}_t \|_2 &= \| \lambda \bm{I} + \sum_{i=1}^t \vect(\nabla f_{\bm{W}^{(i-1)}}(\bm{x}_{i, a_i}) ) \cdot \vect(\nabla f_{\bm{W}^{(i-1)}}(\bm{x}_{i, a_i}) )^T / m\|_2 \\
    &\leq \lambda + \sum_{i=1}^t \| \nabla f_{\bm{W}^{(i-1)}}(\bm{x}_{i, a_i}) \|_F^2 / m \leq \lambda + C_3 ^2 tL,
\end{align*}
where the first inequality is by triangle's inequality and the second inequality is by $\|\bm{W}^{(i)} - \bm{W}^{(0)} \|_F \leq \sqrt{\frac{i}{m\lambda }}, \forall i \in [n]$ and Lemma \ref{lemma:bounded_gradient}. 
\end{proof}

\subsection{Proof of Lemma \ref{lemma:bound_confidence_direction}}
\begin{proof}[Proof of Lemma \ref{lemma:bound_confidence_direction}]
Let $\mathcal{E}(\delta)$ be the event in which the following $(n+2)$ events hold simultaneously: the events in which Lemma \ref{lemma:grad_difference_away_from_init} for each $\omega \in \left \{ \sqrt{\frac{i}{ m \lambda L}} : 1 \leq i \leq n \right\}$ holds, the event in which Lemma \ref{lemma:empirical_ntk} for $\epsilon = (nK)^{-1}$ holds, and the event in which Lemma \ref{lemma:bounded_gradient} holds. 

Under event $\mathcal{E}(\delta)$, we have 
\begin{align*}
    \| \nabla f_{\bm{W}^{(t-1)}}(\bm{x}_{t,a_t} ) \cdot m^{-1/2} \|_F \leq C_3 L^{1/2}, \forall t \in [n].
\end{align*}
Thus, by \citep[Lemma~11]{DBLP:conf/nips/Abbasi-YadkoriPS11}, if we choose $\lambda \geq \max \{1, c_6^2 L \}$, we have 
\begin{align*}
    \sum_{i=1}^t  \| \nabla f_{\bm{W}^{(i-1)}}(\bm{x}_{i,a_i}  ) \cdot m^{-1/2} \|_{ \bm{\Lambda}^{-1}_{i-1} }^2
    &\leq 2 \log \frac{\det (\bm{\Lambda}_t)}{ \det (\lambda \bm{I})}. 
\end{align*}

For the second part of Lemma \ref{lemma:bound_confidence_direction}, for any $t \in [n]$, we define 
\begin{align*}
    \bar{\bm{M}}_t &= m^{-1/2} \cdot [ \vect(\nabla f_{\bm{W}^{(0)}}(\bm{x}_{1, a_1}) ), \ldots, \vect(\nabla f_{\bm{W}^{(0)}}(\bm{x}_{t, a_t}) ) ] \in \mathbb{R}^{p \times t}, \\
    \bm{M}_t &= m^{-1/2} \cdot [ \vect(\nabla f_{\bm{W}^{(t-1)}}(\bm{x}_{1, a_1}) ), \ldots, \vect(\nabla f_{\bm{W}^{(t-1)}}(\bm{x}_{t, a_t}) ) ] \in \mathbb{R}^{p \times t}.
\end{align*}
We have $\bar{\bm{\Lambda}}_t = \lambda \bm{I} + \bar{\bm{M}}_t \bar{\bm{M}}_t^T$ and $\bm{\Lambda}_t = \lambda \bm{I} + \bm{M}_t \bm{M}_t^T$, and 
\begin{align*}
     &\bigg | \log \frac{\text{det} (\bm{\Lambda}_t)}{ \text{det}(\lambda \bm{I})} - \log \frac{\text{det} (\bar{\bm{\Lambda}}_t)}{ \text{det}(\lambda \bm{I})} \bigg | = | \log \det( \bm{I} + \bm{M}_t \bm{M}_t^T / \lambda) - \log \det( \bm{I} + \bar{\bm{M}}_t \bar{\bm{M}}_t^T / \lambda) | \\
     &= | \log \det( \bm{I} + \bm{M}_t^T \bm{M}_t  / \lambda) - \log \det( \bm{I} + \bar{\bm{M}}_t^T \bar{\bm{M}}_t  / \lambda) | \\ 
     &\leq \max\{ \| \langle  (\bm{I} + \bm{M}_t^T \bm{M}_t  / \lambda)^{-1}, \bm{M}_t^T \bm{M}_t - \bar{\bm{M}}_t^T \bar{\bm{M}}_t \rangle \|, \| \langle  (\bm{I} + \bar{\bm{M}}_t^T \bar{\bm{M}}_t  / \lambda)^{-1}, \bm{M}_t^T \bm{M}_t - \bar{\bm{M}}_t^T \bar{\bm{M}}_t \rangle \| \} \\ 
     &= \| \langle  (\bm{I} + \bm{M}_t^T \bm{M}_t  / \lambda)^{-1}, \bm{M}_t^T \bm{M}_t - \bar{\bm{M}}_t^T \bar{\bm{M}}_t \rangle \|\\
     &\leq \| (\bm{I} + \bm{M}_t^T \bm{M}_t  / \lambda)^{-1} \|_F \cdot \| \bm{M}_t^T \bm{M}_t - \bar{\bm{M}}_t^T \bar{\bm{M}}_t \|_F \\ 
     &\leq \sqrt{t} \| (\bm{I} + \bm{M}_t^T \bm{M}_t  / \lambda)^{-1} \|_2 \cdot \| \bm{M}_t^T \bm{M}_t - \bar{\bm{M}}_t^T \bar{\bm{M}}_t \|_F \\ 
     &\leq \sqrt{t}  \| \bm{M}_t^T \bm{M}_t - \bar{\bm{M}}_t^T \bar{\bm{M}}_t \|_F \\
     &\leq \sqrt{t} t \max_{1 \leq i,j \leq t} m^{-1} \cdot \bigg | \langle \nabla f_{\bm{W}^{(t-1)}}(\bm{x}_{i, a_i} ), \nabla f_{\bm{W}^{(t-1)}}(\bm{x}_{j, a_j} ) \rangle - \langle \nabla f_{\bm{W}^{(0)}}(\bm{x}_{i, a_i} ), \nabla f_{\bm{W}^{(0)}}(\bm{x}_{j, a_j} ) \rangle
     \bigg | \\ 
     &\leq t^{3/2} m^{-1} \max_{1 \leq i,j \leq t} \bigg( \bigg | \langle \nabla f_{\bm{W}^{(t-1)}}(\bm{x}_{i, a_i} ) - \nabla f_{\bm{W}^{(0)}}(\bm{x}_{i, a_i} ), \nabla f_{\bm{W}^{(t-1)}}(\bm{x}_{j, a_j} ) \rangle 
     \bigg | \\
     &+ \bigg | \langle \nabla f_{\bm{W}^{(t-1)}}(\bm{x}_{j, a_j} ) - \nabla f_{\bm{W}^{(0)}}(\bm{x}_{j, a_j} ), \nabla f_{\bm{W}^{(t-1)}}(\bm{x}_{i, a_i} ) \rangle 
     \bigg |
     \bigg) \\
     &\leq 2 t^{3/2} m^{-1} \max_{i} \| \nabla f_{\bm{W}^{(t-1)}}(\bm{x}_{i, a_i} ) - \nabla f_{\bm{W}^{(0)}}(\bm{x}_{i, a_i} ) \|_F \cdot \max_{i} \| \nabla f_{\bm{W}^{(t-1)}}(\bm{x}_{i, a_i} )  \|_F \\
     &\leq 2 C_2 C_3^2 t^{3/2} m^{-1/6} (\log m)^{1/2} L^{23/6} \lambda^{-1/6} 
\end{align*}
where the second equality is by $\det(\bm{I} + \bm{A} \bm{A}^T) = \det(\bm{I} + \bm{A}^T \bm{A})$, the first inequality is by that $\log\det$ is {concave}, {the third equality is assumed without loss of generality}, the second inequality is by that $\langle \bm{A}, \bm{B} \rangle \leq \| \bm{A} \|_F \| \bm{B} \|_F$, the third inequality is by that $\| \bm{A}\|_F \leq \sqrt{t} \| \bm{A}\|_2$ for $\bm{A} \in \mathbb{R}^{t \times t}$, the fourth inequality is by that  $\bm{I} + \bm{M}_t^T \bm{M}_t  / \lambda \succeq \bm{I}$, the fifth inequality is by that $\| \bm{A} \|_F \leq t \| \bm{A} \|_{\infty} $ for $\bm{A} \in \mathbb{R}^{t \times t}$, the sixth inequality is by the triangle inequality, and the last inequality is by Lemma \ref{lemma:grad_difference_away_from_init}, and Lemma \ref{lemma:bounded_gradient}. 

The third part of Lemma \ref{lemma:bound_confidence_direction} directly follows the argument in \citep[(B.18)]{zhou2019neural} and uses Lemma \ref{lemma:empirical_ntk} for $\epsilon = (nK)^{-1}$. 

Finally, it is easy to verify that the condition of $m$ in Lemma \ref{lemma:bound_confidence_direction} satisfies the condition of $m$ for $\mathcal{E}(\delta/(n+2))$, and union bound we have $\mathbb{P}(\mathcal{E}(\delta/(n+2))) \geq 1 - \delta$.
\end{proof}

\renewcommand{\thesection}{D}
\section{Details of baseline algorithms in Section \ref{section:experiment}}
\label{section:algo_baseline}
In this section, we present the details of each representative baseline methods and of the B-mode version of NeuraLCB in Section \ref{section:experiment}. We summarize the baseline methods in Table \ref{tab:summary_baseline}.

\begin{table}[h!]
    \caption{A summary of the baseline methods.}
    \label{tab:summary_baseline}
    \centering
    \begin{tabular}{c|c|c|c|}
    \hline
       Baseline  & Algorithm & Type & Function approximation \\
       \hline 
       \hline
        LinLCB & Algorithm \ref{alg:linlcb} & Pessimism & Linear \\
        KernLCB & Algorithm \ref{alg:kernlcb} & Pessimism & RKSH \\ 
        NeuralLinLCB & Algorithm \ref{alg:neurallinlcb} & Pessimism & Linear w/ fixed neural features \\ 
        NeuralLinGreedy & Algorithm \ref{alg:neurallingreedy} & Greedy & Linear w/ fixed neural features \\ 
        NeuralGreedy & Algorithm \ref{alg:neuralgreedy} & Greedy & Neural networks \\ 
        \hline
    \end{tabular}
\end{table}

\begin{algorithm}[h]
\caption{\textbf{LinLCB}}
\begin{algorithmic}[1]

\Require  Offline data $\mathcal{D}_n = \{(\bm{x}_t, a_t, r_t)\}_{t=1}^n$, reg. parameter $\lambda > 0$, confidence parameters $\beta > 0$.

\State $\bm{\Lambda}_n \leftarrow \lambda \bm{I} + \sum_{t=1}^n \bm{x}_{t, a_t} \bm{x}_{t, a_t}^T $
\State $\hat{\bm{\theta}}_n \leftarrow \bm{\Lambda}_n^{-1} \sum_{t=1}^n \bm{x}_{t,a_t} r_t$
\State $L(\bm{u}) \leftarrow \langle \hat{\bm{\theta}}_n, \bm{u} \rangle - \beta \| \bm{u} \|_{\bm{\Lambda}_n^{-1}}, \forall \bm{u} \in \mathbb{R}^d$
\Ensure  $\hat{\pi}(\bm{x}) \leftarrow \argmax_{a \in [K]} L(\bm{x}_a)$
\end{algorithmic}
\label{alg:linlcb}
\end{algorithm}

\begin{algorithm}[h]
\caption{\textbf{KernLCB}}
\begin{algorithmic}[1]

\Require  Offline data $\mathcal{D}_n = \{(\bm{x}_t, a_t, r_t)\}_{t=1}^n$, reg. parameter $\lambda > 0$, confidence parameters $\beta > 0$, kernel $k: \mathbb{R}^d \times \mathbb{R}^d \rightarrow \mathbb{R}$.
\State $\bm{k}_n(\bm{u}) \leftarrow [k(\bm{u}, \bm{x}_{1,a_1}), \ldots, k(\bm{u}, \bm{x}_{n,a_n})]^T, \forall \bm{u} \in \mathbb{R}^d$
\State $\bm{K}_n \leftarrow [k(\bm{x}_{i,a_i}, \bm{x}_{j,a_j})]_{1 \leq i,j \leq n}$
\State $\bm{y}_n \leftarrow [r_1, \ldots, r_n]^T$
\State $L(\bm{u}) \leftarrow \bm{k}_n(\bm{u})^T (\bm{K}_n + \lambda I)^{-1} \bm{y}_n - \beta \sqrt{k(\bm{u},\bm{u}) - \bm{k}_n^T(\bm{u}) (\bm{K}_n + \lambda I)^{-1} \bm{k}_n}, \forall \bm{u} \in \mathbb{R}^d$
\Ensure  $\hat{\pi}(\bm{x}) \leftarrow \argmax_{a \in [K]} L(\bm{x}_a)$ 
\end{algorithmic}
\label{alg:kernlcb}
\end{algorithm}

\begin{algorithm}[h!]
\caption{\textbf{NeuralLinLCB}}
\begin{algorithmic}[1]

\Require  Offline data $\mathcal{D}_n = \{(\bm{x}_t, a_t, r_t)\}_{t=1}^n$, regularization parameter $\lambda > 0$, confidence parameters $\beta > 0$.

\State Initialize the same neural network with the same initialization scheme as in NeuraLCB to obtain the neural network function $f_{\bm{W}^{(0)}}$ at initialization
\State $\phi(\bm{u}) \leftarrow \vect(\nabla f_{\bm{W}^{(0)}}(\bm{u})) \in \mathbb{R}^p, \forall \bm{u} \in \mathbb{R}^d$
\State $\bm{\Lambda}_n \leftarrow \lambda \bm{I} + \sum_{t=1}^n \phi(\bm{x}_{t, a_t}) \phi(\bm{x}_{t, a_t})^T $
\State $\hat{\bm{\theta}}_n \leftarrow \bm{\Lambda}_n^{-1} \sum_{t=1}^n \phi(\bm{x}_{t, a_t}) r_t$
\State $L(\bm{u}) \leftarrow \langle \hat{\bm{\theta}}_n, \phi(\bm{u}) \rangle - \beta \| \phi(\bm{u}) \|_{\bm{\Lambda}_n^{-1}}, \forall \bm{u} \in \mathbb{R}^d$
\Ensure  $\hat{\pi}(\bm{x}) \leftarrow \argmax_{a \in [K]} L(\bm{x}_a)$
\end{algorithmic}
\label{alg:neurallinlcb}
\end{algorithm}

\begin{algorithm}[h!]
\caption{\textbf{NeuralLinGreedy}}
\begin{algorithmic}[1]

\Require  Offline data $\mathcal{D}_n = \{(\bm{x}_t, a_t, r_t)\}_{t=1}^n$, regularization parameter $\lambda > 0$.

\State Initialize the same neural network with the same initialization scheme as in NeuraLCB to obtain the neural network function $f_{\bm{W}^{(0)}}$ at initialization
\State $\phi(\bm{u}) \leftarrow \vect(\nabla f_{\bm{W}^{(0)}}(\bm{u})) \in \mathbb{R}^p, \forall \bm{u} \in \mathbb{R}^d$
\State $\hat{\bm{\theta}}_n \leftarrow \bm{\Lambda}_n^{-1} \sum_{t=1}^n \phi(\bm{x}_{t, a_t}) r_t$
\State $L(\bm{u}) \leftarrow \langle \hat{\bm{\theta}}_n, \phi(\bm{u}) \rangle, \forall \bm{u} \in \mathbb{R}^d$
\Ensure  $\hat{\pi}(\bm{x}) \leftarrow \argmax_{a \in [K]} L(\bm{x}_a)$ 
\end{algorithmic}
\label{alg:neurallingreedy}
\end{algorithm}

\begin{algorithm}[h!]
\caption{\textbf{NeuralGreedy}}
\begin{algorithmic}[1]

\Require  Offline data $\mathcal{D}_n = \{(\bm{x}_t, a_t, r_t)\}_{t=1}^n$, step sizes $ \{\eta_t\}_{t=1}^n$ , regularization parameter $\lambda > 0$.

\State Initialize $\bm{W}^{(0)}$ as follows: set $\bm{W}_l^{(0)} = [\bar{\bm{W}}_l, \hspace{0.1cm} \bm{0};  \bm{0}, \hspace{0.1cm} \bar{\bm{W}}_l ], \forall l \in [L-1]$ where each entry of $\bar{\bm{W}}_l$ is generated independently from $\mathcal{N}(0, 4/m)$, and set $\bm{W}_L^{(0)} = [\bm{w}^T, -\bm{w}^T]$ where each entry of $\bm{w}$ is generated independently from $\mathcal{N}(0, 2/m)$. 


\For{$t = 1, \ldots, n$}
\State Retrieve $(\bm{x}_t, a_t, r_t)$ from $\mathcal{D}_n$. 
\State $\hat{\pi}_t(\bm{x}) \leftarrow \argmax_{a \in [K]}f_{\bm{W}^{(t-1)}}(\bm{x}_a ), \forall \bm{x}$. 


\State $\bm{W}^{(t)} \leftarrow \bm{W}^{(t-1)} - \eta_t \nabla \mathcal{L}_t(\bm{W}^{(t-1)})$ where $\mathcal{L}_t(\bm{W}) = \frac{1}{2}(f_{\bm{W}}(\bm{x}_{t,a_t}) - r_t)^2 + \frac{m \lambda}{2} \| \bm{W} - \bm{W}^{(0)} \|^2_F$. 
\EndFor

\Ensure  Randomly sample $\hat{\pi}$ uniformly from $\{\hat{\pi}_1, \ldots, \hat{\pi}_n\}$. 
\end{algorithmic}
\label{alg:neuralgreedy}
\end{algorithm}

\begin{algorithm}[h!]
\caption{\textbf{NeuraLCB (B-mode)}}
\begin{algorithmic}[1]

\Require  Offline data $\mathcal{D}_n = \{(\bm{x}_t, a_t, r_t)\}_{t=1}^n$, step sizes $ \{\eta_t\}_{t=1}^n$ , regularization parameter $\lambda > 0$, confidence parameters $\{ \beta_t\}_{t=1}^n$, batch size $B > 0$, epoch number $J > 0$. 

\State Initialize $\bm{W}^{(0)}$ as follows: set $\bm{W}_l^{(0)} = [\bar{\bm{W}}_l, \hspace{0.1cm} \bm{0};  \bm{0}, \hspace{0.1cm} \bar{\bm{W}}_l ], \forall l \in [L-1]$ where each entry of $\bar{\bm{W}}_l$ is generated independently from $\mathcal{N}(0, 4/m)$, and set $\bm{W}_L^{(0)} = [\bm{w}^T, -\bm{w}^T]$ where each entry of $\bm{w}$ is generated independently from $\mathcal{N}(0, 2/m)$. 

\State $\bm{\Lambda}_0 \leftarrow \lambda \bm{I}$. 

\For{$t = 1, \ldots, n$}
\State Retrieve $(\bm{x}_t, a_t, r_t)$ from $\mathcal{D}_n$. 
\State $L_t(\bm{u}) \leftarrow f_{\bm{W}^{(t-1)}}(\bm{u} ) - \beta_{t-1} \| \nabla f_{\bm{W}^{(t-1)}}(\bm{u} ) \cdot m^{-1/2} \|_{ \bm{\Lambda}^{-1}_{t-1} }, \forall \bm{u} \in \mathbb{R}^d$ 
\State $\hat{\pi}_t(\bm{x}) \leftarrow \argmax_{a \in [K]} L_t(\bm{x}_a)$, for all $\bm{x} = \{\bm{x}_a \in \mathbb{R}^d: a \in [K]\}$. 

\State $\bm{\Lambda}_t \leftarrow \bm{\Lambda}_{t-1} + \text{vec}(\nabla f_{\bm{W}^{(t-1)}}(\bm{x}_{t, a_t}) ) \cdot \text{vec}(\nabla f_{\bm{W}^{(t-1)}}(\bm{x}_{t, a_t}) )^T / m $. 

\State $\tilde{\bm{W}}^{(0)} \leftarrow \bm{W}^{(t-1)}$
\For{$j = 1, \ldots, J$}

    \State Sample a batch of data $B_t = \{\bm{x}_{t_q, a_{t_q}}, r_{t_q}\}_{q=1}^B$ from $\mathcal{D}_t$
    
    \State $\mathcal{L}_t^{(j)}(\bm{W}) \leftarrow \sum_{q =1}^B \frac{1}{2 B}(f_{\bm{W}}(\bm{x}_{t_q,a_{t_q}}) - r_{t_q})^2 + \frac{m \lambda}{2} \| \bm{W} - \bm{W}^{(0)} \|^2_F$
    
    \State $\tilde{\bm{W}}^{(j)} \leftarrow \tilde{\bm{W}}^{(j-1)} - \eta_t \nabla \mathcal{L}_t^{(j)}(\tilde{\bm{W}}^{(j-1)})$

\EndFor

\State $\bm{W}^{(t)} \leftarrow \tilde{\bm{W}}^{(J)} $
\EndFor

\Ensure  Randomly sample $\hat{\pi}$ uniformly from $\{\hat{\pi}_1, \ldots, \hat{\pi}_n\}$. 
\end{algorithmic}
\label{alg:neuralcb-sgd-bmode}
\end{algorithm}

\newpage
\renewcommand{\thesection}{E}
\section{Datasets}
\label{sec:datasets}
We present a detailed description about the UCI datasets used in our experiment. 
\begin{itemize}
    \item \textbf{Mushroom data}: Each sample represents a set of attributes of a mushroom. There are two actions to take on each mushroom sample: eat or no eat. Eating an editable mushroom generates a reward of $+5$ while eating
    a poisonous mushroom yields a reward of $+5$ with probability $0.5$ and a reward of $-35$ otherwise. No eating gives a reward of $0$. 
    \item \textbf{Statlog data}: The shuttle dataset contains the data about a space shuttle flight where the goal is to predict the state of the radiator subsystem of the shuttle. There are total $K=7$ states to predict where approximately $80\%$ of the data belongs to one state. A learner receives a reward of $1$ if it selects the correct state and $0$ otherwise. 
    \item \textbf{Adult data}: The Adult dataset contains personal information from the US Census Bureau database. Following \citep{riquelme2018deep}, we use the $K=14$ different occupations as actions and $d = 94$ covariates as contexts. As in the Statlog data, a learner obtains a reward of $1$ for making the right prediction, and $0$ otherwise.
    
    \item {MNIST data: The MNIST data contains images of various handwritten digits from $0$ to $9$. We use $K = 10$ different digit classes as actions and $d= 784$ covariates as contexts. As in the Statlog and Adult data, a learner obtains a reward of $1$ for making the right prediction, and $0$ otherwise. }
\end{itemize}
We summarizes the statistics of the above datasets in Table \ref{tab:data_statistics}.
\begin{table}[t]
    \caption{The real-world dataset statistics}
    \label{tab:data_statistics}
    \centering
    \begin{tabular}{c|c|c|c|c|}
    \hline 
      Dataset   & Mushroom & Statlog & Adult & \textcolor{red}{MNIST}  \\
     \hline
     \hline 
      Context dimension & 22 & 9 & 94 & 784 \\
      
      Number of classes & 2 & 7 & 14 & 10 \\ 
      
      Number of instances & 8,124 & 43,500 & 45,222 & 70,000\\
      \hline
    \end{tabular}
\end{table}

\renewcommand{\thesection}{F}
\section{Additional Experiments}
\label{sec:add_expr}
In this section, we complement the experimental results in the main paper with additional experiments regarding the learning ability of our algorithm on dependent data and the different behaviours of S-mode and B-mode training. 
\subsection{Dependent data}
As the sub-optimality guarantee in Theorem \ref{main_theorem} does not require the offline policy to be stationary, we evaluate the empirical performance of our algorithm and the baselines on a new setup of offline data collection that represents dependent actions. In particular, instead of using a stationary policy to collect offline actions as in Section \ref{section:experiment}, here we used an adaptive policy $\mu$ defined as 
\begin{align*}
    \mu(a|\mathcal{D}_{t-1}, \bm{x}_t) = (1-\epsilon) * \pi^*(a|\bm{x}_t) + \epsilon * \pi_{\text{LinUCB}}(a | \mathcal{D}_{t-1}, \bm{x}_t), 
\end{align*}
where $\pi^*$ is the optimal policy and $\pi_{\text{LinUCB}}$ is the linear UCB learner \citep{DBLP:conf/nips/Abbasi-YadkoriPS11}. This weighted policy makes the collected offline data $a_t$ dependent on $\mathcal{D}_{t-1}$ while making sure that the offline data has a sufficient coverage over the data of the optimal policy, as LinUCB does not perform well in several non-linear data. We used $\epsilon = 0.9$ in this experiment. 

\begin{figure}[h!]
\centering
\begin{minipage}[p]{0.48\linewidth}
\centering
\includegraphics[width=\linewidth]{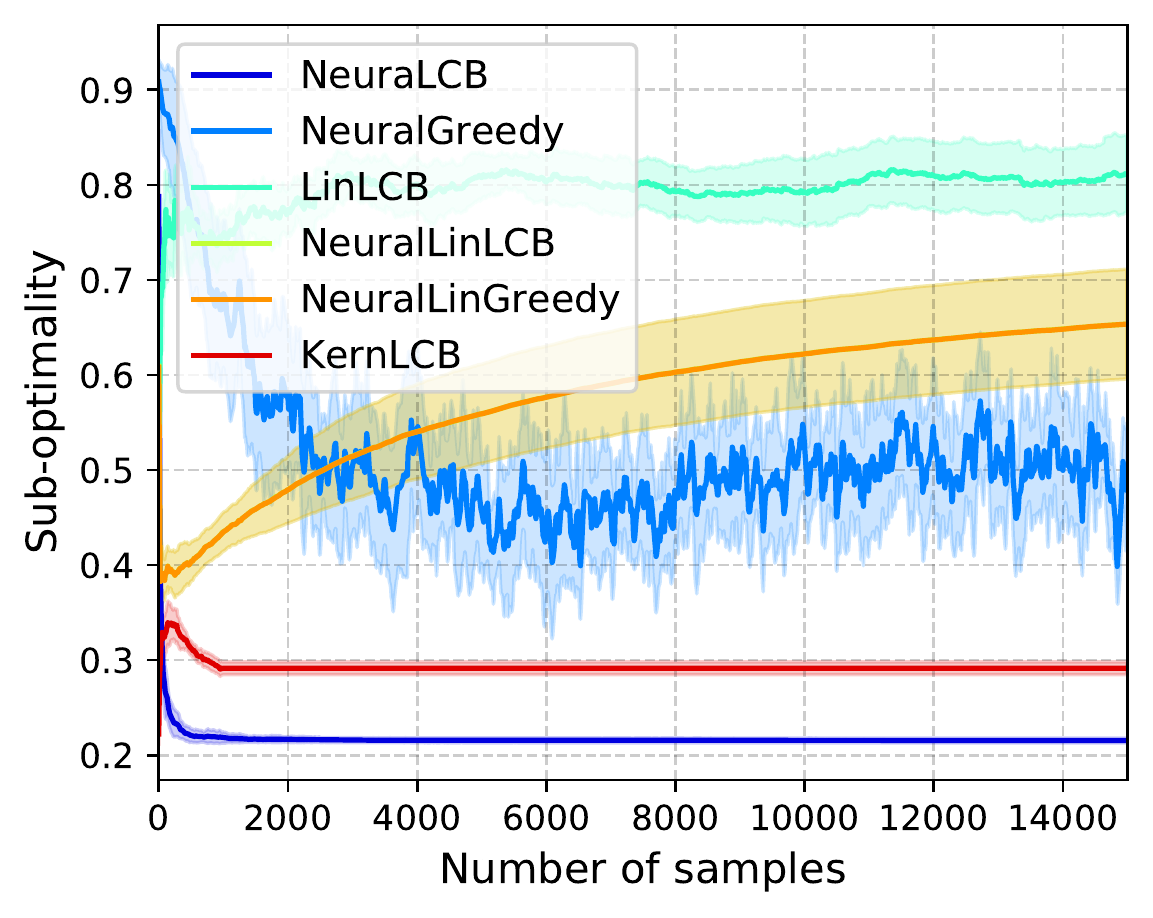}
\subcaption{Statlog} 
\label{fig:first}
\end{minipage}
\begin{minipage}[p]{0.48\linewidth}
\centering
\includegraphics[width=\linewidth]{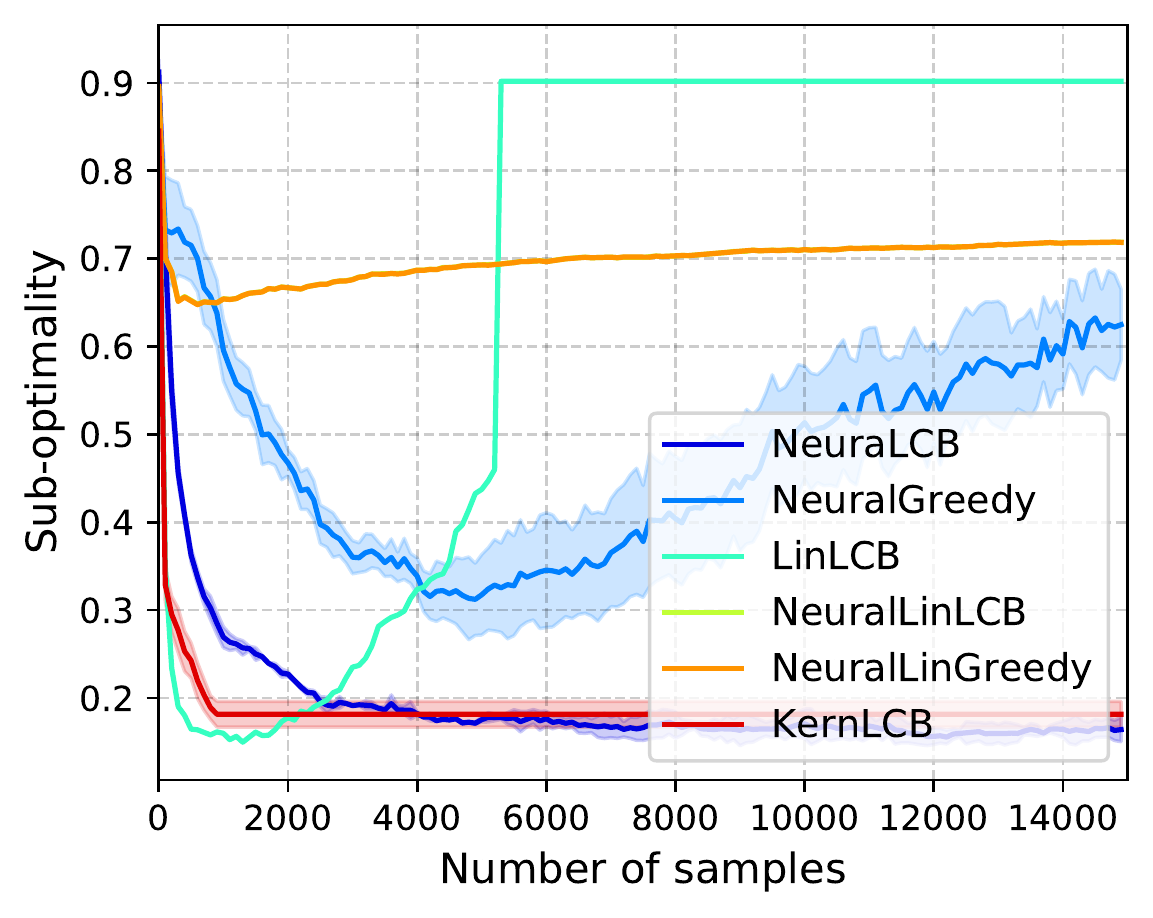}
\subcaption{MNIST}
\end{minipage}
\caption{The sub-optimality of NeuraLCB versus the baseline algorithms on real-world datasets with correlated structures.}
\label{fig:dependent_data}
\end{figure}

The results on Statlog and MNIST are shown in Figure \ref{fig:dependent_data}. We make two important observations. First, on this dependent data, the baseline methods with linear models (LinLCB, NeuralLinLCB and NeuralLinGreedy) \footnote{KernLCB also does not change its performance much, but its computational complexity is still an major issue.} perform almost as bad as when they learn on the independent data in Figure \ref{fig:realworld}, suggesting that linear models are highly insufficient to learn complex rewards in real-world data, regardless of how offline data were collected. Secondly, the main competitor of our method, NeuralGreedy suffers an apparent performance degradation (especially in a larger dataset like MNIST) while NeuraLCB maintains a superior performance, suggesting the effectiveness of pessimism in our method on dealing with offline data and the robustness of our method toward dependent data.

\subsection{S-mode versus B-mode training}
As in Section \ref{section:experiment} we implemented two different training modes: S-mode (Algorithm \ref{alg:neuralcb-sgd}) and B-mode (Algorithm \ref{alg:neuralcb-sgd-bmode}). We compare the empirical performances of S-mode and B-mode on various datasets. As this variant is only applicable to NeuralGreedy and NeuraLCB, we only depict the performances of these algorithms. The results on Cosine (synthetic dataset), Statlog, Adult and MNIST are shown in Figure \ref{fig:mode_training}. 

We make the following observations, which are consistent with the conclusion in the main paper while giving more insights. While the B-mode outperforms the S-mode on Cosine, the S-mode significantly outperforms the B-mode in all the tested real-world datasets. Moreover, the B-mode of NeuraLCB outperforms or at least is compatible to the S-mode of NeuralGreedy in the real-world datasets. First, these observations suggest the superiority of our method on the real-world datasets. Second, to explain the performance difference of S-mode and B-mode on synthetic and real-world datasets in our experiment, we speculate that the online-like nature of our algorithm tends to reduce the need of B-mode in practical datasets because B-mode in the streaming data might cause over-fitting (as there could be some data points in the past streaming that has been fitted for multiple times). In synthetic and simple data such as Cosine, over-fitting tends to associate with strong prediction of the underlying reward function as the simple synthetic reward function such as Cosine is sufficiently smooth, unlike practical reward functions in the real-world datasets. 


\begin{figure}[h!]
\centering
\begin{minipage}[p]{0.24\linewidth}
\centering
\includegraphics[width=\linewidth]{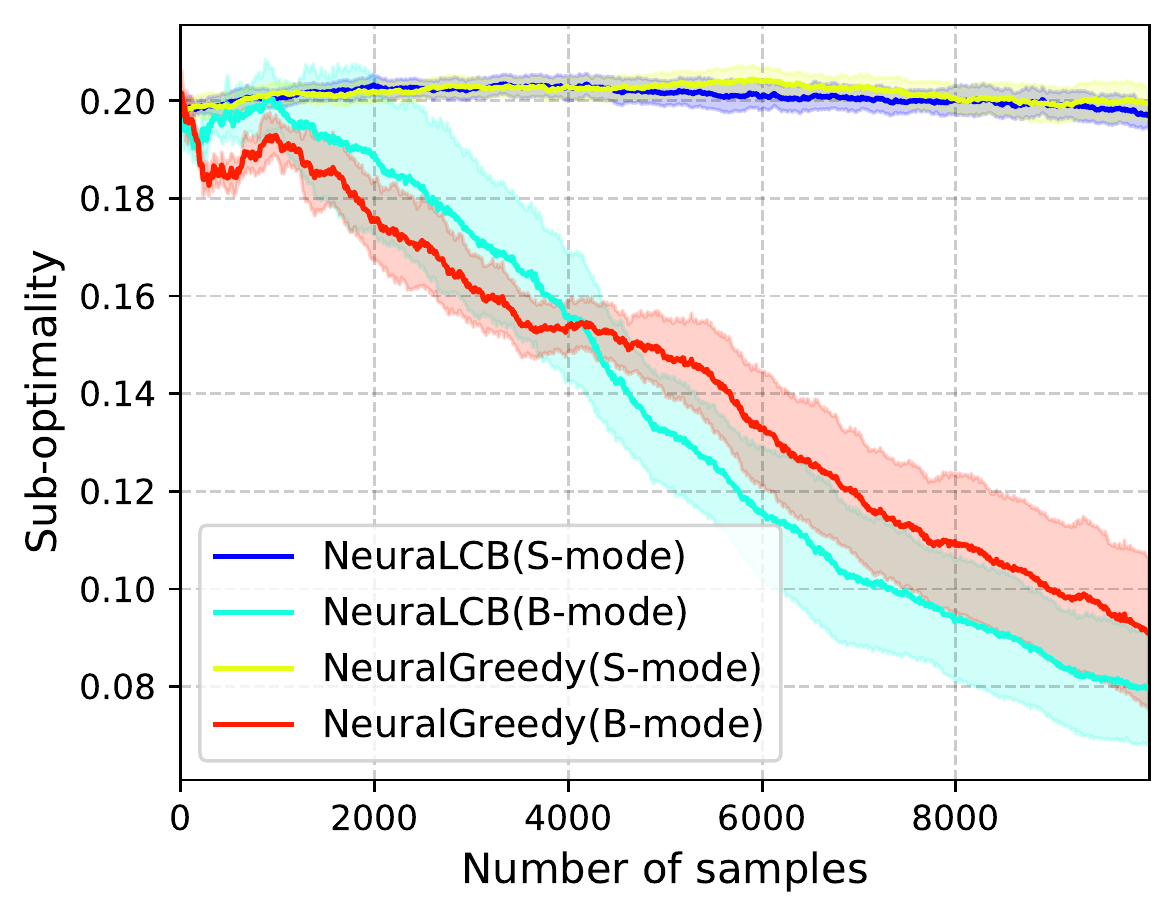}
\subcaption{Cosine} 
\label{fig:first}
\end{minipage}
\begin{minipage}[p]{0.24\linewidth}
\centering
\includegraphics[width=\linewidth]{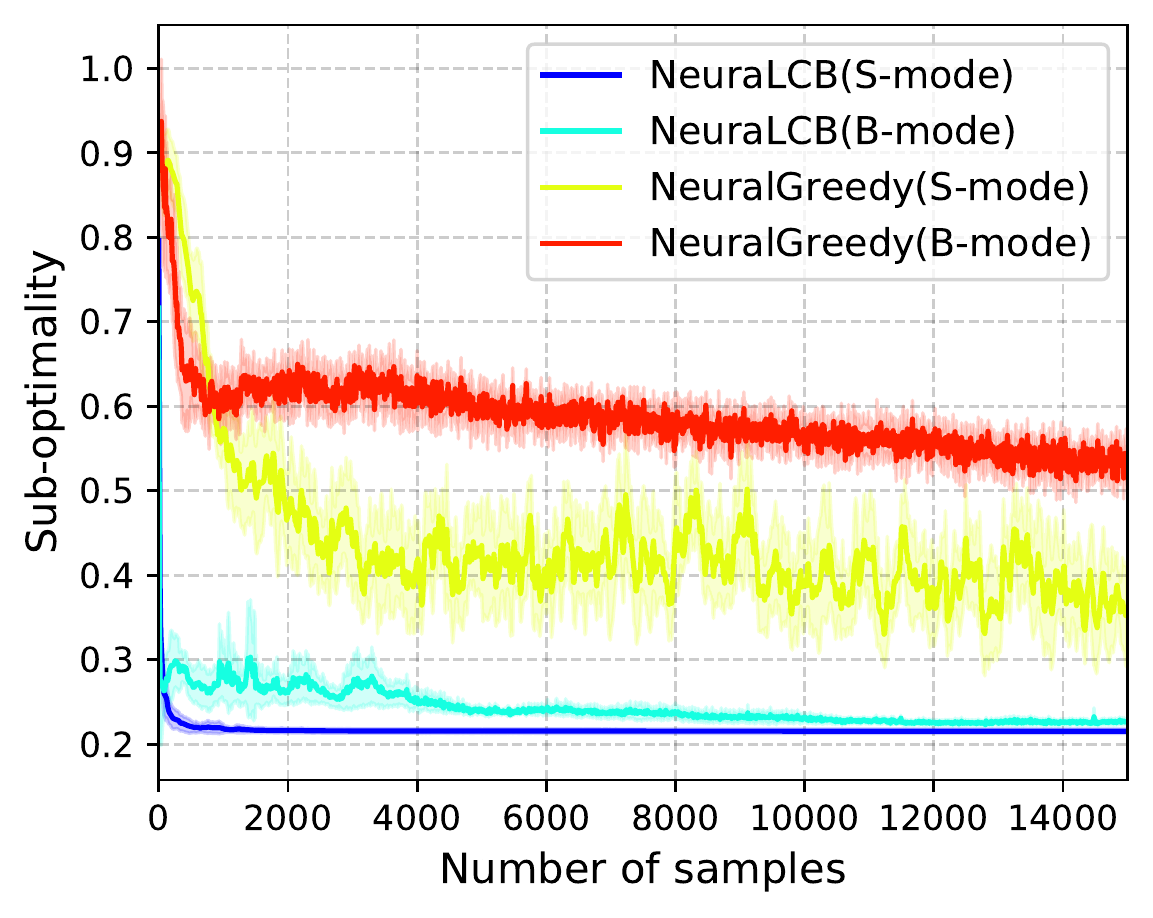}
\subcaption{Statlog}
\end{minipage}
\begin{minipage}[p]{0.24\linewidth}
\centering
\includegraphics[width=\linewidth]{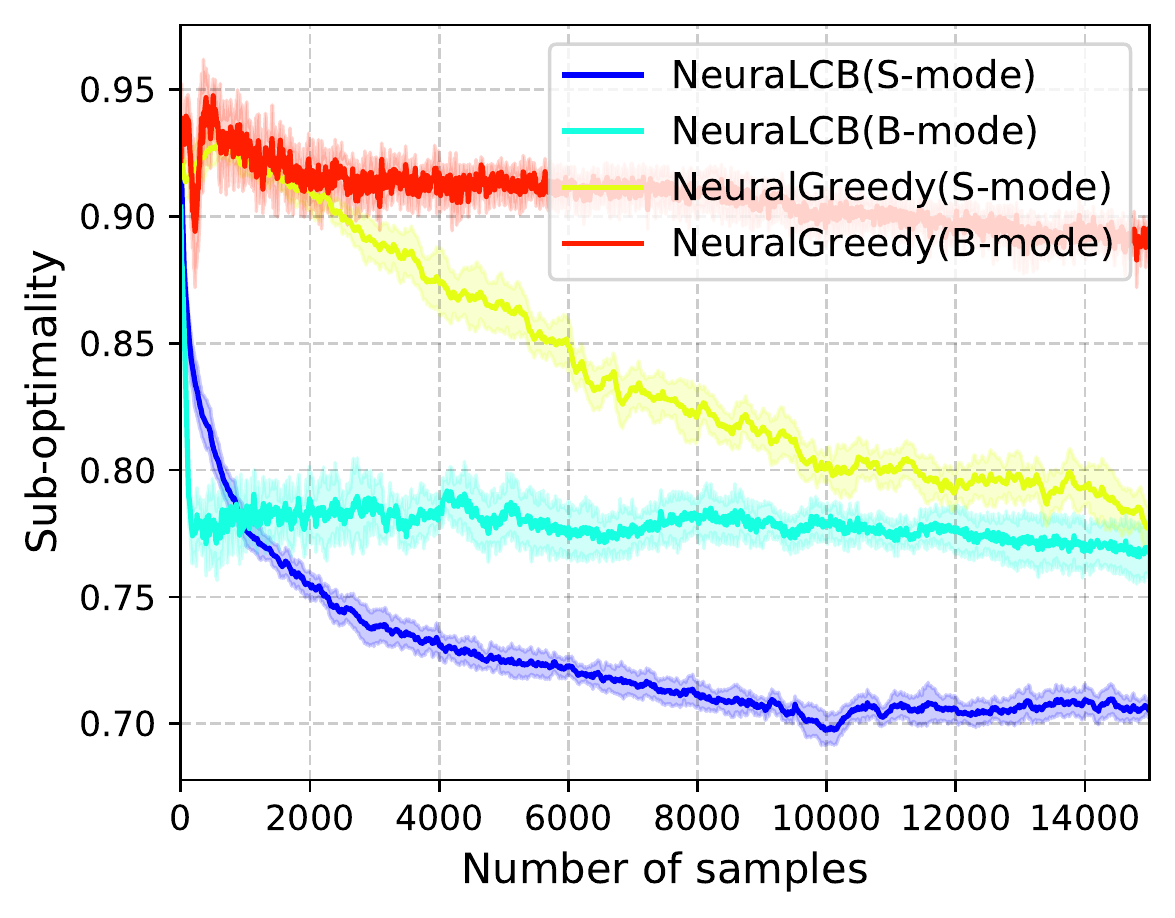}
\subcaption{Adult}
\end{minipage}
\begin{minipage}[p]{0.24\linewidth}
\centering
\includegraphics[width=\linewidth]{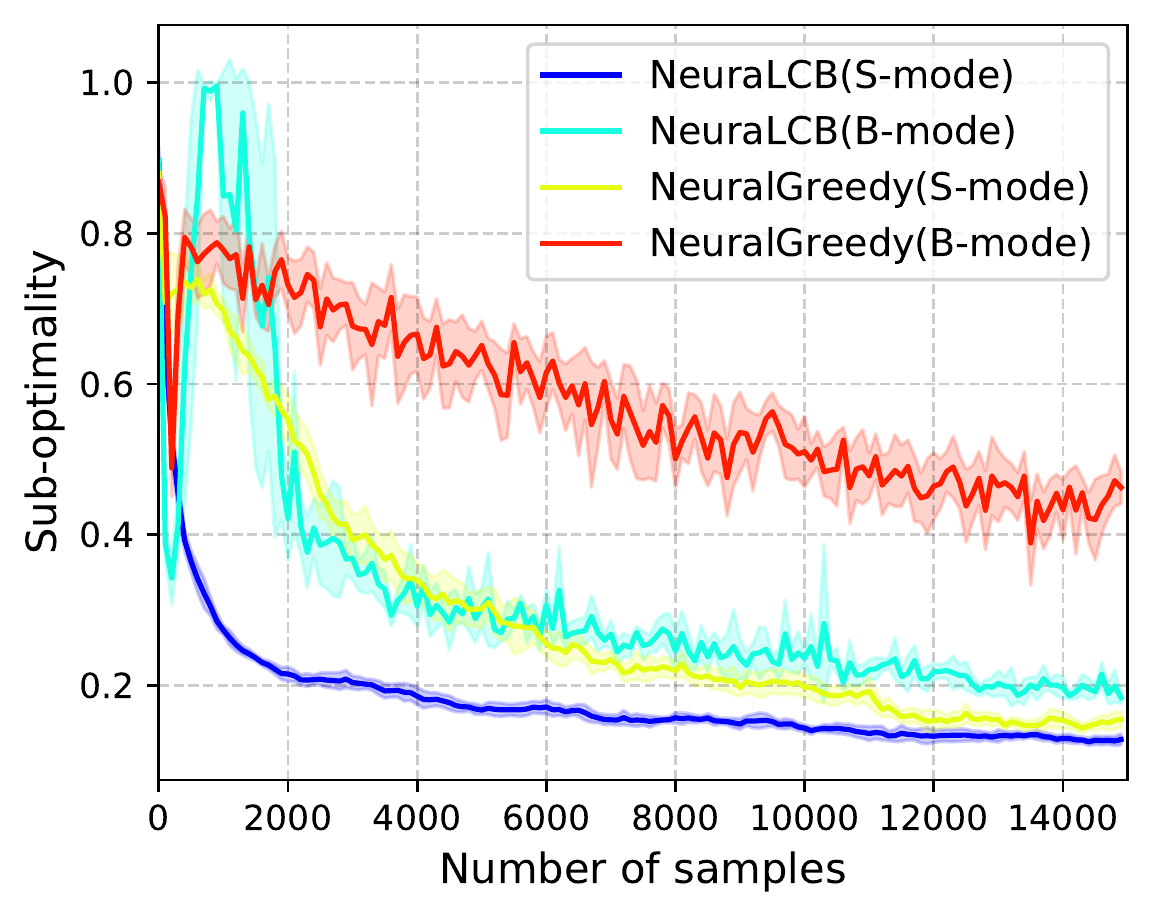}
\subcaption{MNIST}
\end{minipage}
\caption{Comparison of S-mode and B-mode training.}
\label{fig:mode_training}
\end{figure}

\end{document}